\documentclass[11pt]{article}

\usepackage[letterpaper,top=1.15in,bottom=1.15in,left=1.25in,right=1.25in]{geometry}

\usepackage[T1]{fontenc}
\usepackage[utf8]{inputenc}
\usepackage{libertinus}          
\usepackage{libertinust1math}
\usepackage{microtype}

\usepackage{amsmath,amssymb,amsthm,amsfonts}

\usepackage{graphicx}
\usepackage[table,dvipsnames]{xcolor}
\usepackage{booktabs}
\usepackage{array}
\usepackage{colortbl}
\usepackage{caption}
\usepackage{newfloat}
\usepackage{listings}
\usepackage[most]{tcolorbox}
\usepackage{algorithm}
\usepackage{algorithmic}
\usepackage{enumitem}
\usepackage[hyphens]{url}
\usepackage{natbib}

\renewcommand{\cite}{\citep}  
\usepackage{titlesec}
\usepackage{fancyhdr}

\definecolor{ink}{HTML}{1A1A1A}         
\definecolor{accent}{HTML}{A6522C}      
\definecolor{accentdeep}{HTML}{7A3C1F}
\definecolor{cream}{HTML}{F6F3EC}       
\definecolor{creamline}{HTML}{DDD5C6}   
\definecolor{slate}{HTML}{2F5D75}       
\definecolor{slatelt}{HTML}{4A7B93}     
\definecolor{tblhead}{HTML}{EFEBE2}     
\definecolor{rulegrey}{HTML}{9A9384}

\color{ink}

\renewcommand{\arraystretch}{1.25}
\titleformat{\section}
  {\normalfont\sffamily\large\bfseries\color{ink}}
  {\color{accent}\thesection}{0.7em}{}
\titleformat{\subsection}
  {\normalfont\sffamily\normalsize\bfseries\color{ink}}
  {\color{accent}\thesubsection}{0.6em}{}
\titleformat{\subsubsection}
  {\normalfont\sffamily\small\bfseries\color{ink}}
  {\color{accent}\thesubsubsection}{0.5em}{}
\titlespacing*{\section}{0pt}{1.9ex plus .3ex}{0.9ex}
\titlespacing*{\subsection}{0pt}{1.5ex plus .2ex}{0.7ex}
\titleformat{\paragraph}[runin]
  {\normalfont\bfseries\color{ink}}{}{0pt}{}[.]
\titlespacing*{\paragraph}{0pt}{1.3ex plus .25ex}{0.6em}
\DeclareCaptionStyle{ruled}{labelfont=normalfont,labelsep=colon,strut=off}
\floatstyle{ruled}
\newfloat{listing}{tb}{lst}{}
\floatname{listing}{Listing}

\newtcblisting{promptbox}[2][]{%
    enhanced, breakable, width=\textwidth,
    listing only, listing engine=listings,
    title={#2}, title after break={#2 (continued)},
    colback=cream, colframe=creamline,
    colbacktitle=accent, coltitle=white,
    fonttitle=\bfseries\footnotesize\sffamily,
    boxrule=0.5pt, arc=2pt, left=2.5mm, right=2.5mm, top=1.8mm, bottom=1.8mm,
    listing options={basicstyle=\ttfamily\scriptsize, breaklines=true,
      columns=fullflexible, keepspaces=true, showstringspaces=false, numbers=none},
    #1}

\newcommand{\widetab}[1]{\noindent\makebox[\textwidth][c]{#1}}

\renewenvironment{abstract}
  {\begin{tcolorbox}[enhanced, breakable, colback=cream, colframe=cream,
      borderline west={2.5pt}{0pt}{accent}, arc=1pt,
      left=5mm, right=5mm, top=3.5mm, bottom=3.5mm, boxrule=0pt]
   \small
   {\sffamily\bfseries\color{accentdeep}Abstract}\par\vspace{0.6ex}}
  {\end{tcolorbox}\vspace{1ex}}

\newtheorem{proposition}{Proposition}

\newcommand{\appendixbanner}[2]{%
  \clearpage
  \begin{tcolorbox}[enhanced, colback=cream, colframe=creamline,
      borderline west={3pt}{0pt}{accent}, arc=1pt, boxrule=0.5pt,
      left=5mm, right=5mm, top=4mm, bottom=4mm]
    {\sffamily\bfseries\Large\color{accentdeep}#1}\par\vspace{0.8ex}
    {\small #2}
  \end{tcolorbox}
  \vspace{1.5ex}}

\usepackage[colorlinks=true,linkcolor=slate,citecolor=slatelt,urlcolor=accentdeep,
            breaklinks=true,pdfborder={0 0 0}]{hyperref}
\hypersetup{
  pdftitle={Consilience: Conformally Calibrated Communication Control for Hidden-Profile Multi-Agent Reasoning},
  pdfauthor={Abhijith Babu, Ramneet Kaur, Vishal Pramanik, Olivera Kotevska, Nathaniel D. Bastian, Susmit Jha, Sunny Raj, Yanzhao Wu, Sumit Kumar Jha, Anirban Roy},
  pdfkeywords={multi-agent LLM, hidden profile, conformal prediction, communication control}}
\usepackage[nameinlink]{cleveref}

\begin{document}

\begin{center}
  {\sffamily\bfseries\Large\color{ink}
   Consilience: Conformally Calibrated Communication Control\\[0.4em]
   for Hidden-Profile Multi-Agent Reasoning\par}
  \vspace{1.1em}
  {\normalsize
   Abhi{}jith Babu\textsuperscript{1},\;
   Ramneet Kaur\textsuperscript{2},\;
   Vishal Pramanik\textsuperscript{3},\;
   Olivera Kotevska\textsuperscript{4},\;
   Nathaniel D. Bastian\textsuperscript{5}\\[0.4em]
   Susmit Jha\textsuperscript{2},\;
   Sunny Raj\textsuperscript{6},\;
   Yanzhao Wu\textsuperscript{1},\;
   Sumit Kumar Jha\textsuperscript{3,*},\;
   Anirban Roy\textsuperscript{2}\par}
  \vspace{0.9em}
  {\small\color{ink!75}
   \textsuperscript{1}Knight Foundation School of Computing and Information Sciences, Florida International University\\
   \textsuperscript{2}SRI International\quad
   \textsuperscript{3}Department of Computer \& Information Science \& Engineering, University of Florida\\
   \textsuperscript{4}Oak Ridge National Laboratory\quad
   \textsuperscript{5}Army Cyber Institute, United States Military Academy\\
   \textsuperscript{6}Department of Computer Science and Engineering, Oakland University\\[0.5em]
   \textsuperscript{*}Corresponding author: \texttt{sumit.jha@ufl.edu}\par}
  \vspace{1.0em}
  {\color{creamline}\rule{\textwidth}{1.2pt}}
\end{center}
\vspace{0.6em}

\begin{abstract}

Multi-agent LLM systems can improve reasoning by pooling diverse perspectives, but their effectiveness depends on coordinating communication, particularly in hidden-profile settings where each agent holds only part of the evidence required for a correct decision. Existing protocols, including fixed schedules, round-robin exchange, and unstructured debate, provide no guarantee that a conversational action is appropriate. We propose \textit{Consilience}, an inference-time orchestration framework that both steers and certifies multi-agent communication under distributed private information. At each turn, Consilience summarizes the discussion using a compact state capturing uncertainty, disagreement, evidence gain, redundancy, and premature consensus, then selects both a communication intervention (challenge, clarify, seek evidence, or route) and an appropriate speaker. Its central contribution is a round-wise conformal calibration procedure that provides a distribution-free, finite-sample guarantee: at each discussion round, conditional on reaching that round, the one-step regret of a controller's proposed action is bounded by a calibrated threshold with marginal probability at least $1-\alpha$; an acceptance mechanism enforces the same guarantee for the executed action by replacing inadmissible proposals. On HiddenBench-style hidden-profile tasks spanning 12 open and closed weight language models, Consilience improves decision accuracy and communication efficiency over fixed and unstructured discussion protocols, sometimes surpassing a full-information baseline where every agent observes all evidence. These results demonstrate that certified adaptive communication control can be more valuable than increasing information availability, providing a practical mechanism for reliable multi-agent LLM coordination.

\end{abstract}

\section{Introduction}
\label{sec:introduction}

Large language models (LLMs) are increasingly deployed as multi-agent systems in which several agents propose solutions, exchange arguments, critique one another, and jointly produce a decision \cite{du2023improving,liang2024encouraging,xiong2023examining,wu2024autogen,li2023camel}. Such systems have shown promise in mathematical reasoning, question answering, software development, planning, and tool-assisted problem solving. Their appeal stems from the idea that multiple agents can contribute complementary reasoning, identify errors overlooked by individual models, and produce more reliable collective decisions. However, recent studies question whether discussion alone consistently delivers these benefits. Multi-agent deliberation can duplicate the capabilities of a well-prompted single agent, amplify correlated errors, or converge prematurely on persuasive but incorrect conclusions \cite{wang2024rethinking,liang2024encouraging}. Simply increasing the number of agents or discussion rounds therefore does not guarantee better collective reasoning.

This limitation becomes especially consequential when task-relevant information is distributed across agents. Many collaborative decisions in medicine, intelligence analysis, and organizational planning involve participants who observe different pieces of evidence. No individual has access to the complete information required for a correct decision; the group succeeds only if its members communicate and integrate their complementary knowledge. This setting is captured by the \textit{hidden-profile paradigm}, originally introduced in social psychology to study information pooling in group decision making \cite{stasser1985pooling,stasser1988computer}. In a hidden profile, information shared by all group members supports a plausible but suboptimal alternative, whereas uniquely held information, when pooled across members, reveals the correct choice \cite{stasser1985pooling, lu2012twentyfive,schulz2012achieve}.

Although discussion should theoretically recover distributed evidence, human groups consistently exhibit a common-information bias: information shared by several members is more likely to be mentioned, repeated, and reinforced than information known to only one member \cite{stasser1985pooling,stasser1989information,gigone1993common}. As a result, discussion often reinforces initial preferences instead of correcting them. A meta-analysis of hidden-profile studies found that groups discussed substantially more common than unique information and were significantly less likely to identify the correct solution under hidden- than full-information conditions \cite{lu2012twentyfive}. Effective hidden-profile reasoning therefore requires more than allowing participants to speak: groups must surface unshared evidence, consider dissenting perspectives, revise premature beliefs, and determine when consensus is sufficiently supported to justify termination \cite{lu2012twentyfive,schulz2012achieve}.

Recent work shows that LLM-based groups reproduce many of these collective reasoning failures. HiddenBench~\citep{li2025hiddenbench} formalizes the hidden-profile paradigm for multi-agent LLM evaluation using 65 tasks derived from custom scenarios, prior human studies, and automatically generated decision problems. Each agent receives shared task information together with an asymmetric private clue, and the agents must deliberate before making a collective decision. Across 15 models from four language-model families, HiddenBench finds persistent failures to integrate distributed evidence: agents can collectively possess sufficient information yet still converge on an incorrect answer, and model scale or general reasoning strength does not reliably predict successful information pooling \cite{li2025hiddenbench}. HiddenBench, therefore, exposes a gap not directly addressed by the conventional reasoning benchmarks. The bottleneck is not only whether an LLM can reason from the evidence it observes, but whether a group can control its communication so that the right evidence is elicited, routed, and incorporated before consensus forms.

Existing multi-agent LLM methods provide only partial solutions. Debate frameworks improve reasoning through critique and competing arguments \cite{du2023improving,liang2024encouraging,xiong2023examining}; role-based frameworks coordinate specialized agents \cite{li2023camel,wu2024autogen,hong2024metagpt}; and coordination methods rely on predefined interaction topologies, voting, judges, or aggregation mechanisms \cite{chan2023chateval,chen2024agentverse,zhu2025multiagentbench}. These approaches work well when agents share most of the available information. Hidden-profile tasks, however, require adaptive communication control: the system must decide whether to elicit missing evidence, clarify disagreement, challenge unsupported conclusions, route the discussion, or terminate deliberation. Fixed round-robin protocols~\citep{li2025hiddenbench} cannot adapt to the evolving discussion state, while unconstrained LLM orchestrators may make unreliable control decisions.

Moreover, existing debate and orchestration methods generally provide no statistical guarantee that the communication action selected at a given state is appropriate. Such guarantees matter because communication decisions shape the information available to the group: selecting the wrong speaker, repeating shared evidence, or voting too early can reinforce an incorrect consensus, while unnecessary discussion increases cost and risks distraction or context degradation. Communication control should therefore be treated as a sequential decision problem rather than a fixed prompting protocol. At each round, the controller must use the agents' evolving beliefs and public transcript to select an intervention that improves the collective decision state while avoiding unnecessary communication. Our contributions are:

\begin{itemize}

\item \textbf{Closed-loop communication control.} We pose hidden-profile reasoning as a sequential communication-control problem and introduce Consilience, a framework that uses the evolving discussion state to select, conformally certify, and route communication interventions, enabling closed-loop coordination among LLM agents.

\item \textbf{Conformal reliability.} We propose a round-wise conformal calibration framework that certifies communication actions using one-step regret, providing finite-sample, distribution-free regret guarantees and replacing inadmissible proposals with conformally certified alternatives.

\item \textbf{Comprehensive evaluation.} We evaluate Consilience across 12 open- and closed-weight language models on HiddenBench~\cite{li2025hiddenbench} and an LLM-generated benchmark based on the GroupTravelBench task design~\cite{cheng2026grouptravelbench}. We compare against hidden-information, round-robin, and full-information baselines, while analyzing communication efficiency, conformal calibration, speaker routing, controller generalization, heterogeneous agents, and key ablations.

\end{itemize}

\begin{figure}[t]
\centering
\includegraphics[width=0.8\textwidth]{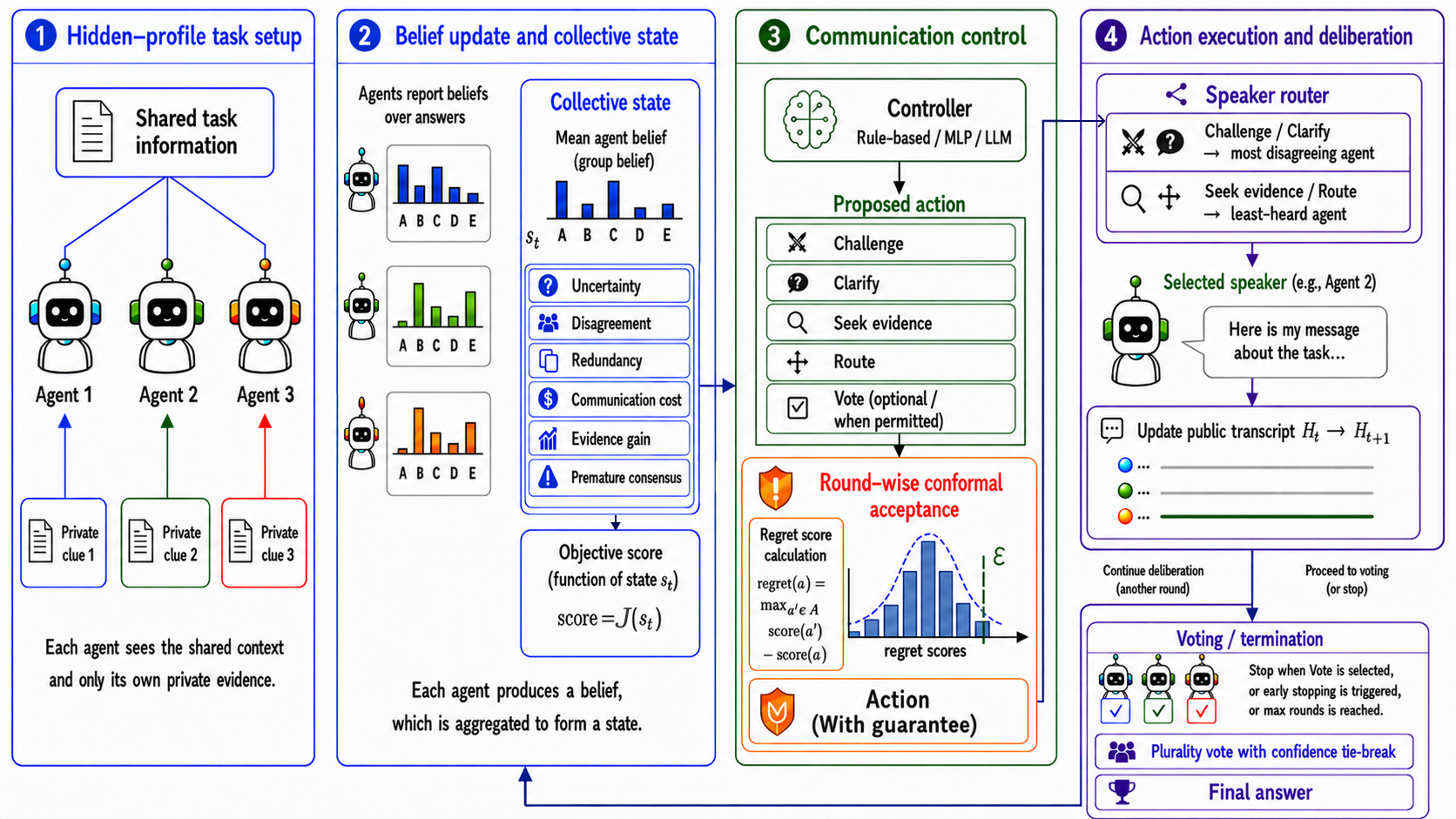} 
\caption{Proposed framework \textit{Consilience}: (1) Each agent receives shared task information and a private clue.
(2) Agents report beliefs conditioned on the public transcript, which
Consilience aggregates with transcript-level signals into the collective
state \(s_t\). (3) A controller proposes a communication action, and
round-wise conformal acceptance retains or replaces it using calibrated
one-step regret. (4) An action-conditional router selects the speaker,
whose message updates the transcript and agent beliefs. The loop repeats
until termination, followed by plurality voting with confidence-based
tie-breaking.}
\label{workflow}
\end{figure}

\section{Problem Formulation}
\label{sec:problem-formulation}

We consider hidden-profile multi-agent decision tasks in which the
information needed to identify the correct answer is distributed across
multiple agents. A task is a tuple
\(
(x,\mathcal{Y},I_{\mathrm{shared}},\{I_i\}_{i=1}^{N})
\),
where \(x\) is the task description,
\(\mathcal{Y}=\{y_1,\ldots,y_K\}\) is a finite answer set,
\(I_{\mathrm{shared}}\) is visible to every agent, and \(I_i\) is visible
only to agent \(i\). Thus, agent \(i\) initially observes
\((x,\mathcal{Y},I_{\mathrm{shared}},I_i)\), but not \(I_j\) for
\(j\neq i\).

The agents communicate through a shared public transcript
\(\mathcal{H}_t=(m_1,\ldots,m_t)\), with
\(\mathcal{H}_0=\varnothing\). Private evidence becomes available to the
group only when an agent communicates it on the public transcript. The goal is
to produce a collective answer \(\widehat{y}\in\mathcal{Y}\) while
eliciting and integrating the relevant private evidence with limited
communication. Consilience treats this process as closed-loop
communication control: the controller chooses what kind of action intervention
is needed, and a separate action-conditional router chooses which agent
should carry it out. Neither the controller nor the router answers the task itself or
generates evidence on an agent's behalf. The final answer is produced from the quorum among the LLM agents from their beliefs formed from the information on the public transcript.

\section{Consilience: Conformally Calibrated Communication for
Multi-Agent Systems}
\label{sec:methodology}

\subsection{Method Overview}
\label{sec:method-overview} Figure~\ref{workflow} summarizes the Consilience loop. 

\paragraph{Design principle.}
Let \(Y\in\mathcal{Y}\) denote the latent correct answer. An ideal
controller would choose the intervention that maximizes the expected
quality of the next collective belief while penalizing communication:
\begin{equation}
\label{eq:ideal-control}
a_t^\star
=
\arg\max_{a\in\mathcal{A}}
\mathbb{E}\!\left[
U(b_{t+1})-\lambda C(\mathcal{H}_{t+1})
\mid s_t,a
\right],
\end{equation}
where \(b_{t+1}\) is the collective posterior induced after the next
message, \(U(b)=-H(b)\) is negative posterior entropy, and
\(C(\mathcal{H})\) is communication cost. Exact optimization of
Equation~\eqref{eq:ideal-control} is intractable because each action can
induce many possible language-model responses and future discussion
trajectories. Consilience, therefore, uses an
observable collective state and a tractable one-step objective to propose
an intervention, round-wise conformal acceptance to screen that proposal,
and an action-conditional router to select its speaker. The resulting
message updates the public transcript and agent beliefs, closing the loop.
The following subsections formalize state construction, action proposal,
conformal acceptance, speaker routing, and termination.

\subsubsection{Stage 1: Belief Elicitation and State Construction}
\label{sec:collective-state}

After observing \(\mathcal{H}_t\), each agent reports
\[
p_i^{(t)} = \left[ p_i^{(t)}(y_1),\ldots,p_i^{(t)}(y_K) \right],
\qquad
\sum_{k=1}^{K}p_i^{(t)}(y_k)=1.
\]
This is the agent's reported belief conditional on
\((x,I_{\mathrm{shared}},I_i,\mathcal{H}_t)\); it is not assumed to be a
calibrated Bayesian posterior. The aggregate belief is the arithmetic
mean
\begin{equation}
\label{eq:aggregate-belief}
\bar p^{(t)} = \frac{1}{N}\sum_{i=1}^{N}p_i^{(t)}.
\end{equation}
Consilience represents the current discussion by
\begin{equation}
\label{eq:state-vector}
s_t = \left[ H_t,D_t,R_t, G_t, C_t, P_t \right],
\end{equation}
whose components are defined as:

\paragraph{Group uncertainty ($H_t$)}
We use the normalized entropy of the aggregate belief:
\[
H_t = -\frac{\sum_{k=1}^{K}\bar p^{(t)}(y_k)\log \bar p^{(t)}(y_k)}{\log K}.
\]
Thus, larger values of $H_t$ indicate a less decisive group belief.

\paragraph{Inter-agent disagreement ($D_t$)}
Aggregate confidence can conceal incompatible individual beliefs. We
measure disagreement by
\[
D_t = \frac{1}{N}\sum_{i=1}^{N}\operatorname{JSD}\!\left( p_i^{(t)}\parallel\bar p^{(t)} \right),
\]
where \(\operatorname{JSD}\) is the Jensen--Shannon divergence
\cite{lin2002divergence}. A large \(D_t\) indicates that some agents
interpret the available evidence differently from the group.

\paragraph{Message redundancy ($R_t$) and evidence gain ($G_t$)}
Let
\[
R_t = \phi_R(m_t,\mathcal{H}_{t-1})
\quad\text{and}\quad
G_t = \phi_G(m_t,\mathcal{H}_{t-1},\mathcal{Y}),
\]
where the fixed estimators \(\phi_R,\phi_G\in[0,1]\) are applied
identically across controllers. The redundancy estimator $\phi_R$ combines
lexical overlap with an LLM assessment of whether \(m_t\) repeats
previously public information. The evidence-gain estimator $\phi_G$ assesses
whether \(m_t\) contributes a previously unshared, decision-relevant
fact that supports or eliminates candidate answers. Higher \(R_t\)
indicates more repetition, whereas higher \(G_t\) indicates a more
informative contribution. We set \(R_0=G_0=0\).

\paragraph{Communication cost $(C_t)$}
We use normalized transcript length:
\[
C_t = \frac{\operatorname{Words}(\mathcal{H}_t)}{1000}.
\]

\paragraph{Premature consensus ($P_t)$}
Let
\(
p_t^\star=\max_{y\in\mathcal{Y}}\bar p^{(t)}(y)
\).
We distinguish supported agreement from early anchoring as
\[
P_t
=
p_t^\star(1-G_t)
\,
\mathbf{1}\!\left[
p_t^\star>0.70,\;
G_t<0.30,\;
t<K+2
\right].
\]
The score is nonzero when the group becomes highly confident early
in the discussion despite weak recent evidence.

\paragraph{Surrogate discussion loss.}
The state variables define
\begin{equation*}
\label{eq:surrogate-objective}
J(s_t)
=
\alpha_H H_t
+\alpha_D D_t
+\alpha_R R_t
-\alpha_G G_t
+\alpha_C C_t
+\alpha_P P_t.
\end{equation*}
Lower values are preferred. The coefficients are fixed design
hyperparameters, selected before evaluation and held constant across
controller variants and test tasks. The objective is a local control
signal, not a substitute for answer correctness. To summarize, Stage 1: the current transcript and private agent contexts
produce the observable state \(s_t\) and its scalar discussion loss
\(J(s_t)\). Stage 2 uses these to propose the next intervention.

\subsubsection{Stage 2: Communication-Action Proposal}
\label{sec:actions-regret}

\paragraph{Communication-action space}
Consilience uses four communication interventions:
\[
\mathcal{A}_{\mathrm{comm}}
=
\{
\textsc{Challenge},
\textsc{Clarify},
\textsc{SeekEvidence},
\textsc{Route}
\}.
\]
Controllers that can terminate adaptively additionally include
\textsc{Vote}. All controller variants use the same action definitions
and message templates.

\begin{itemize}
    \item \textsc{Challenge} asks the speaker to examine the current
    leading answer, identify unsupported assumptions, and provide
    potentially contradictory evidence.

    \item \textsc{Clarify} asks an agent with a divergent belief to explain the private evidence/interpretation for that disagreement.

    \item \textsc{SeekEvidence} requests one previously unshared,
    decision-relevant fact and an explanation of how it supports or
    eliminates candidate answers.

    \item \textsc{Route} is the evidence-sharing intervention.
    It asks an under-participating agent for concise discriminative
    facts.

    \item \textsc{Vote}, when available, terminates deliberation and
    invokes the final voting rule.
\end{itemize}

\textbf{Controller variants.}
Every controller maps {${s_t}$} from Stage 1 to a proposal action
\(
a_{t,0}=\pi_m(s_t)
\).
The routing and message-generation procedures remain fixed, so
controller comparisons isolate action proposal.

\paragraph{Rule-based controller}
The interpretable controller applies fixed thresholds to the state
variables. It selects \textsc{Challenge} when confidence is high but
recent evidence is weak, \textsc{Clarify} when disagreement is high,
\textsc{SeekEvidence} when communication is redundant or uninformative,
and \textsc{Route} otherwise. It permits voting only after a minimum
number of rounds, when the leading belief is sufficiently strong,
disagreement is low, and premature consensus is absent.

\paragraph{Learned controller}
Exhaustive training rollouts evaluate every admissible action from each encountered state using the Stage-3 counterfactual reward and record the vector
\[
\left[ r_t(s_t,a_1),\ldots,r_t(s_t,a_{|\mathcal{A}_{\mathrm{comm}}|}) \right].
\]
A multilayer perceptron
\(
f_\theta:\mathbb{R}^{6}\rightarrow
\mathbb{R}^{|\mathcal{A}_{\mathrm{comm}}|}
\)
with one 12-unit ReLU hidden layer is trained to predict this vector.
At inference time, it proposes
\[
\pi_{\mathrm{MLP}}(s_t) = \arg\max_{a\in\mathcal{A}_{\mathrm{comm}}} f_\theta(s_t)_a.
\]

\paragraph{LLM-based controller}
The LLM controller receives the task, candidate answers, public transcript, state vector, agent beliefs, and speaker counts, and selects one action from the admissible action space. A second variant includes \textsc{Vote}, enabling adaptive termination. Speaker selection remains separate. Regardless of controller type, Stage 2 outputs only the proposed action
\(a_{t,0}\). Stage 3 evaluates that proposal against the other
admissible actions and determines the action \(a_t\) that may actually
be executed.

\subsubsection{Stage 3: Counterfactual Evaluation and Conformal Acceptance}
\label{sec:conformal-control}

\paragraph{Conformal prediction}
\citet{vovk2005algorithmic} provide distribution-free, finite-sample
marginal coverage for exchangeable non-conformity scores on calibration datapoints $(Q_1,\ldots, Q_n)$ and the test input $Q_{test}$. Specifically, if
\(Q_1,\ldots,Q_n,Q_{\mathrm{test}}\) are exchangeable and are produced
by the same fixed scoring procedure, then, for
\(k=\lceil(n+1)(1-\alpha)\rceil\), and \(\epsilon\) as the \(k\)-th
smallest calibration score, we have the following probabilistic guarantees:
$\label{eq:split-conformal-background}
\mathbb{P}\!\left(Q_{\mathrm{test}}\leq\epsilon\right)
\geq 1-\alpha.$

\paragraph{Counterfactual action outcomes}
For a nonterminal action \(a\) applied at state \(s_t\), the fixed
Stage-4 router
selects a speaker, the speaker generates a message, all agents update
their beliefs, and Consilience obtains a candidate next state
\(s_{t+1}^{a}\). All candidate actions are evaluated independently from
the same frozen pre-action state; only the finally accepted branch is
committed to the public transcript. The one-step improvement is
$
\label{eq:action-reward}
r_t(s_t,a)
=
J(s_t)-J(s_{t+1}^{a}).$
We propose one-step regret as the non-conformity score:
\begin{align*}
\label{eq:action-regret}
\mathcal{R}_t(s_t,a)
=
\max_{a'\in\mathcal{A}_t^m(s_t)}
r_t(s_t,a')
-
r_t(s_t,a)
\\ =
J(s_{t+1}^{a})
-
\min_{a'\in\mathcal{A}_t^m(s_t)}
J(s_{t+1}^{a'}),
\end{align*}
where \(\mathcal{A}_t^m(s_t)\) is the set of actions admissible for
controller \(m\) at round \(t\). Hence
\(\mathcal{R}_t(s_t,a)\geq 0\), and at least one admissible action has
regret zero. Terminal actions, when available,
are evaluated using the same fixed terminal-state scoring rule in
calibration and testing.

\paragraph{Offline round-wise calibration}
For a fixed controller method \(m\), define the non-conformity score of
its proposal at round \(t\) as
\begin{equation}
\label{eq:nonconformity}
Q_t^m(s)
=
\mathcal{R}_t\!\left(s,\pi_m(s)\right).
\end{equation}
Calibration uses complete trajectories generated by the fixed
controller, routing policy, model, prompts, and score-generating
procedure. Because some trajectories terminate early, only calibration
trajectories that reach round \(t\) contribute to the round-\(t\)
multiset
\[
\mathcal{Q}_{\mathrm{cal}}^{m,(t)}
=
\left\{
Q_t^m(s_t^\tau):
\tau\in\mathcal{D}_{\mathrm{cal}}
\text{ reaches round }t
\right\}.
\]
Let \(n_t=|\mathcal{Q}_{\mathrm{cal}}^{m,(t)}|\) and $
k_t
=
\left\lceil(n_t+1)(1-\alpha)\right\rceil.
$
The threshold \(\epsilon_t^m\) is the
\(k_t\)-th smallest calibration score. 

\paragraph{Online conformal acceptance}
At test time, non-conformity score (or regret) for the controller's proposed action is computed, and if it is within the conformal threshold $\epsilon_t^m$, then it is passed to the router for stage 4. Otherwise, the near-optimal conformal action
set $\Gamma_t^m(s)$ is generated {by executing all possible actions}:
\begin{equation}
\label{eq:conformal-action-set}
\Gamma_t^m(s)
=
\left\{
a\in\mathcal{A}_t^m(s):
\mathcal{R}_t(s,a)\leq\epsilon_t^m
\right\}.
\end{equation}
Here, Consilience applies a fixed fallback
\(\operatorname{Select}\), which returns a minimum-regret member of
\(\Gamma_t^m(s)\), with ties resolved by a fixed action ordering:
\begin{equation}
\label{eq:fallback-select}
\operatorname{Select}(\Gamma_t^m(s)) \in \arg\min_{a\in\Gamma_t^m(s)}\mathcal{R}_t(s,a).
\end{equation}
The set is nonempty because a minimum-regret action has regret zero and
\(\epsilon_t^m\geq 0\).

Constructing either the calibration score in
Equation~\eqref{eq:nonconformity} or the test-time set in
Equation~\eqref{eq:conformal-action-set} requires evaluating every
admissible action from the same state. These evaluations are
counterfactual branches: they are used to calculate regret, and only the
accepted branch updates the actual transcript. Stage 3 therefore maps the proposal \(a_{t,0}\) to an accepted action
\(a_t\). The action has not yet changed the real transcript: Stage 4
first chooses the agent who will execute it.

\subsubsection{Stage 4: Action-Conditioned Routing and Transcript Update}
\label{sec:speaker-routing}

After conformal acceptance chooses \(a_t\), the deterministic router
selects speaker \(i_t=\rho(a_t,s_t)\). For disagreement-focused actions,
\begin{equation}
\label{eq:disagreement-routing}
\begin{aligned}
\rho (a_t,s_t) &= \arg \max_i \operatorname{JSD} \!\left (p_i^{(t)} \parallel \bar p^{(t)} \right ) , \\
a_t &\in \{\textsc{Challenge}, \textsc{Clarify}\} .
\end{aligned}
\end{equation}
The most disagreeing agent is the most likely to hold evidence or an
interpretation not yet incorporated into the group belief. For evidence-acquisition actions,
\begin{equation}
\label{eq:evidence-routing}
\rho(a_t,s_t)
=
\arg\min_i n_i^{(t)},
\\
a_t\in\{\textsc{SeekEvidence},\textsc{Route}\},
\end{equation}
where \(n_i^{(t)}\) is the number of messages previously contributed by
agent \(i\). The least-heard agent is used as a proxy for the agent most
likely to retain unshared evidence. Ties in
Equations~\eqref{eq:disagreement-routing}--\eqref{eq:evidence-routing}
are resolved by a fixed agent ordering.

The executed control is therefore
$
u_t=(a_t,i_t),
$ where conformal communication control determines \(a_t\) and the router
determines \(i_t\). The selected agent generates one action-conditioned
message
\[
m_{t+1}
\sim
\operatorname{LLM}_i
\!\left(
x,\mathcal{Y},I_{\mathrm{shared}},I_i,
\mathcal{H}_t,a_t
\right),
\]
which is appended to the public transcript. All agents subsequently
report \(p_i^{(t+1)}\), and Consilience recomputes \(s_{t+1}\). Thus,
if no stopping condition is met, Stage 4 closes the feedback
loop by returning the updated transcript to Stage 1.

\subsubsection{Stage 5: Termination and Collective Prediction}
\label{sec:termination}

Deliberation terminates when \textsc{Vote} is selected by a controller
that admits it, a fixed early-stopping condition is met, or the maximum
number of rounds \(T\) is reached. Each agent then independently reports
an answer \(\widehat y_i\in\mathcal{Y}\) and confidence
\(c_i\in[0,1]\). We define
\begin{equation*}
N(y)
=
\sum_{i=1}^{N}
\mathbf{1}[\widehat y_i=y], \text{ and}
\end{equation*}

\begin{equation*}
\mathcal{Y}_{\mathrm{tie}}
=
\left\{
y:N(y)=\max_{y'\in\mathcal{Y}}N(y')
\right\}.
\end{equation*}

\(\mathcal{Y}_{\mathrm{tie}}\) contains all the plurality winners.

The final prediction is
\[
\widehat y_{\mathrm{final}} = \arg\max_{y\in\mathcal{Y}_{\mathrm{tie}}}\sum_{i:\widehat y_i=y}c_i.
\]

\subsection{End-to-End Procedure and Theoretical Guarantee}
\label{sec:roundwise-guarantee}

The five stages above specify one online control round:
\[
\underbrace{(\mathcal{H}_t,\{I_i\}_{i=1}^{N})}_{\text{current information}}
\xrightarrow{\text{Stage 1}}
s_t
\xrightarrow{\text{Stage 2}}
a_{t,0}
\xrightarrow{\text{Stage 3}}
a_t
\xrightarrow{\text{Stage 4}}
\mathcal{H}_{t+1},
\]
followed by either another round or the Stage-5 prediction. The
calibration phase precedes online deliberation and supplies the
round-specific thresholds used in Stage 3.

\paragraph{Round-wise marginal guarantee.}
Here, we formalize the statistical guarantee for the Stage-3 acceptance rule based on the theory of conformal prediction framework.

Proposition~\ref{prop:roundwise-coverage} specializes this
construction to the round-specific one-step regret scores of a fixed
Consilience controller and gives the sharper finite-sample result for
almost surely distinct scores.

\begin{proposition}[Round-Wise Marginal Regret Coverage]
\label{prop:roundwise-coverage}
Fix a controller \(m\) and round \(t\). Let \(n_t\) calibration
trajectories reach round \(t\), with states
\(s_1,\ldots,s_{n_t}\), and let \(s_{\mathrm{test}}\) be the round-\(t\)
state of a new trajectory conditional on that trajectory reaching round
\(t\). Assume that, conditional on the realized calibration-cohort size
and on the test trajectory reaching round \(t\),
$
Q_t^m(s_1),\ldots,Q_t^m(s_{n_t}),
Q_t^m(s_{\mathrm{test}})
$
are exchangeable, and that \(\pi_m\) and all components of the
score-generating procedure were fixed independently of the calibration
sample. With \(k_t\) and \(\epsilon_t^m\) defined in Stage 3,
\begin{equation}
\label{eq:roundwise-coverage}
\mathbb{P}\!\left(
Q_t^m(s_{\mathrm{test}})\leq\epsilon_t^m
\;\middle|\;
\mathcal{T}_{\mathrm{test}}\text{ reaches }t,\,
N_t=n_t
\right)
\geq 1-\alpha.
\end{equation}
Equivalently,
\begin{equation*}
\label{eq:roundwise-set-coverage}
\mathbb{P}\!\left(
\pi_m(s_{\mathrm{test}})
\in\Gamma_t^m(s_{\mathrm{test}})
\;\middle|\;
\mathcal{T}_{\mathrm{test}}\text{ reaches }t,\,
N_t=n_t
\right)
\geq 1-\alpha.
\end{equation*}
If \(k_t\leq n_t\) and the pooled calibration and test scores
are almost surely distinct, the probability in
Equation~\eqref{eq:roundwise-coverage} equals
\(
k_t/(n_t+1)
\), and therefore lies in
\[
\left[
1-\alpha,\;
1-\alpha+\frac{1}{n_t+1}
\right).
\]
\label{prop:roundwise_coverage}
\end{proposition}
The interpretation and scope of Proposition~\ref{prop:roundwise_coverage} are in the Appendix.

\section{Experiments}

\begin{table}[t]
\centering
\small
\caption{Task and vote accuracy on HiddenBench~\cite{li2025hiddenbench} across the evaluated language models after conformal acceptance. Each entry reports task accuracy / vote accuracy. Hidden Pre, Hidden Post, and Full Info are reproduced from the baseline. Higher values are better. The best values are shown in bold. }
\label{tab:conformal-main}
\widetab{\begin{tabular}{lccccccc}
\toprule
\rowcolor{tblhead}
\textbf{Model} &
\textbf{Hidden Pre} &
\textbf{Hidden Post} &
\textbf{Full Info} &
\textbf{Rules (Ours) } &
\textbf{MLP (Ours) } &
\textbf{LLM (Ours) } &
\textbf{LLM+Vote (Ours) } \\
\midrule
Qwen3-0.6B   & .159/.202 & .185/.203 & \textbf{.359/.364} & .289/.309 & .267/.280 & .333/.331 & .222/.234 \\
Qwen3-1.7B   & .062/.095 & .092/.092 & .267/.285 & \textbf{.644/.629} & .578/.583 & .556/.571 & .556/.560 \\
Qwen3-4B     & .092/.095 & .097/.124 & .369/.406 & .733/.743 & .778/.800 & .778/.783 & \textbf{.844/.846} \\
Qwen3-8B     & .062/.126 & .149/.175 & .631/.596 & \textbf{.911/.926} & .867/.880 & .911/.903 & .867/.880 \\
Qwen3-14B    & .123/.161 & .267/.266 & .774/.752 & .889/.891 & .778/.794 & .867/.863 & \textbf{.933/.931} \\
Qwen3-32B    & .108/.141 & .282/.287 & .826/.817 & \textbf{.933/.937} & .889/.897 & .889/.886 & .867/.869 \\

Qwen2.5-32B  & .046/.087 & .297/.290 & .908/.881 & .911/.897 & .867/.863 & \textbf{.956/.960} & .933/.914 \\
Llama-3.1-8B & .077/.177 & .354/.354 & .610/.599 & .800/.743 & \textbf{.822/.811} & .711/.680 & .622/.623 \\

Phi-4        & .077/.150 & .195/.221 & .877/.819 & \textbf{.933/.914} & .822/.823 & .911/.897 & .911/.914 \\
Mistral-24B  & .072/.170 & .467/.485 & .908/.859 &\textbf{ .956/.943} & .867/.863 & .889/.880 & .867/.891 \\
DeepSeek-v4-Flash & .092/.185 & .384/.395 & .907/.881 & \textbf{.956/.947} & .844/.840 & .867/.863 & .955/.931 \\
DeepSeek-v4-Pro   & .077/.158 & .431/.474 & .954/.933 & \textbf{1.00/1.00} & .867/.874 & .956/.954 &\textbf{ 1.00/1.00} \\
\bottomrule
\end{tabular}}
\end{table}

\begin{table}[t]
\centering
\small
\setlength{\tabcolsep}{8pt}
\caption{Task accuracy and vote accuracy for the benchmark generated from GroupTravelBench \cite{cheng2026grouptravelbench} averaged across 12 models.}
\label{tab:mean-results}
\begin{tabular}{lcc}
\toprule
\rowcolor{tblhead}
Method & Task Accuracy & Vote Accuracy \\
\midrule
Hidden Pre     & 0.2537 & 0.2475 \\
Hidden Post    & 0.4982 & 0.4925 \\
Full Info       & 0.3657 & 0.3358 \\
Rules (Ours)          & 0.6544 & 0.6542 \\
MLP (Ours)            & 0.5071 & 0.5025 \\
LLM (Ours)            & 0.6150 & 0.6125 \\
\textbf{LLM+Vote (Ours) } & \textbf{0.6682} & \textbf{0.6642} \\
\bottomrule
\end{tabular}
\end{table}

We evaluate Consilience through four research questions: \textbf{RQ1} whether adaptive communication improves accuracy over no-discussion and round-robin baselines; \textbf{RQ2} whether distributed deliberation can match full-information reasoning; \textbf{RQ3} how controller choice, communication actions, routing, and termination affect performance; and \textbf{RQ4} whether learned coordination generalizes to unseen tasks and heterogeneous agent groups.

\paragraph{Benchmarks and models.}
We evaluate Consilience on the 65 hidden-profile decision problems in
HiddenBench~\citep{li2025hiddenbench} and on a separately generated
benchmark following the task structure of
GroupTravelBench~\citep{cheng2026grouptravelbench}. Each task distributes
the evidence required for the correct answer across agents, which
communicate only through the public transcript. Because the original
GroupTravelBench data are unavailable, our generated benchmark is not a
reproduction and its results are not directly comparable; generation
and validation details are in the Appendix. We test 12 instruction-tuned
open-weight models spanning multiple families and scales \cite{touvron2023llama,bai2023qwen,abdin2024phi, xu2026deepseek,jiang2023mistral}; model and
serving details, together with details on supplementary closed-model experiments,
are also provided in the Appendix.

\subsection{Compared Discussion Protocols}

We compare four Consilience controllers
(\textsc{Rules}, \textsc{MLP}, \textsc{LLM}, and
\textsc{LLM+Vote}) with three references:
\textsc{Hidden-Pre}, in which agents vote without communicating;
\textsc{Hidden-Post}, which uses fixed round-robin discussion; and
\textsc{Full-Info}, in which agents receive all evidence but do not
discuss. \textsc{Full-Info} is an information-rich empirical reference,
not a theoretical upper bound.

\subsection{Evaluation Metrics}

We report two complementary evaluation metrics computed from the agents' final votes.

\textbf{Vote accuracy.}
Vote accuracy measures the proportion of individual agent votes that match the ground-truth answer:
\[
\mathrm{VoteAcc} = \frac{\sum_{d\in\mathcal{D}}\sum_{i=1}^{N_d}\mathbb{I}[v_{d,i}=y_d]}
                        {\sum_{d\in\mathcal{D}}N_d},
\]
where $\mathcal{D}$ is the evaluation set, $N_d$ is the number of agents for task $d$, $v_{d,i}$ is the final vote of agent $i$, and $y_d$ is the correct answer.

\textbf{Task accuracy.}
Task accuracy measures whether the collective decision is correct. For each task, the final prediction is obtained by plurality voting,
\[
\hat{y}_d = \arg\max_{a\in\mathcal{A}_d}\sum_{i=1}^{N_d}\mathbb{I}[v_{d,i}=a],
\]
with ties broken using the summed confidence of supporting agents when available, and otherwise by the benchmark's deterministic ordering. Task accuracy is then
\[
\mathrm{TaskAcc} = \frac{1}{|\mathcal{D}|}\sum_{d\in\mathcal{D}}\mathbb{I}[\hat{y}_d=y_d].
\]

Task accuracy is our primary evaluation metric, while vote accuracy provides a complementary measure of how consistently individual agents converge to the correct answer.

\section{Results and Discussion}
\label{sec:results}

We provide the results of two benchmarks along with their analysis in this section. Detailed results are in the appendix.

\subsection{Overall Performance}
\label{sec:overall-results}

Table~\ref{tab:conformal-main} reports the accuracy of each discussion protocol
across the evaluated models. The hidden pre-discussion condition provides a
lower-information baseline in which agents cannot exchange their private
evidence. Fixed round-robin discussion improves information availability, but
does not explicitly control which evidence should be requested, which agent
should provide it, or when the discussion should terminate.

Across the evaluated models, Consilience consistently improves over these uncontrolled hidden-information protocols. The strongest controller obtains an average accuracy of \textbf{0.83}, compared with \textbf{0.26} for fixed round-robin discussion, corresponding to an absolute improvement of \textbf{0.57}. This result supports \textbf{RQ1:} the observed gains arise from adapting the communication intervention to the evolving collective state rather than from allowing agents to communicate.

\subsection{Can Communication Compensate for Distributed Information?}
\label{sec:full-information-comparison}

We next compare controlled hidden-information deliberation with the
full-information, no-discussion condition. This comparison separates two
potential sources of performance: access to evidence and the ability to
integrate that evidence effectively.

The best Consilience policy reaches or exceeds the full-information condition on \textbf{11 of the 12 }models. Averaged across models, its accuracy bettered the
full-information condition by \textbf{0.13 }. In almost all cases, controlled deliberation performs better despite no individual agent having direct access to all hidden facts.

This result should not be interpreted as exceeding a theoretical upper bound.
The full-information condition is an information-rich empirical baseline:
language models may still overlook evidence, confuse relationships among
facts, or commit prematurely to an answer. The result instead suggests that
staged evidence sharing and explicit disagreement resolution can facilitate
evidence integration more effectively than presenting all facts in a single
context.

\subsection{Effect of Controller Design}
\label{sec:controller-comparison}

All controller variants share the same actions, routing rules, agents, and prompts, differing only in action selection. The rule-based controller achieves the highest average accuracy \textbf{0.83}, followed by the LLM \textbf{0.80}, LLM+Vote \textbf{0.80}, and MLP \textbf{0.77}, showing that the collective-state variables support effective coordination even without learned control. The two LLM controllers behave differently across model scales as seen in Tab~\ref{tab:conformal-main}: adaptive voting benefits smaller models by stopping before repeated or misleading discussion causes conversational drift, whereas larger models benefit from longer deliberation, which helps uncover additional evidence without degrading context.

\subsection{Generalization to Unseen Tasks}
\label{sec:unseen-results}

To evaluate held-out generalization, we train the MLP controller on five-task subsets and test it on the remaining tasks; the all-task MLP is retained only as an in-domain oracle-imitation diagnostic. The unseen-task MLP achieves \textbf{0.73} average accuracy, compared with \textbf{0.77} for the all-task model, indicating that most of the learned coordination behavior transfers to unseen discussion states. This suggests that the compact state captures recurring conditions such as disagreement, redundancy, and unsupported consensus rather than task-specific content. The remaining gap may reflect task-distribution shift, variation in answer-set size, or imperfect counterfactual reward targets. Full training and evaluation details are provided in the appendix.

\subsection{Ablation Studies}
\label{sec:ablations}

We conduct five ablations that independently examine: (1) random sampling from
conformal prediction sets, (2) action-conditional speaker routing, (3) the
contribution of each communication action, (4) heterogeneous agent
compositions, and (5) controller-state features. In each case, all other
experimental components remain fixed. All detailed results are provided in appendix.

\subsubsection{Conformal Prediction-Set Sampling}
\label{sec:conformal-random-sampling}
tests whether conformally admissible
actions also support successful deliberation. At each discussion round, we uniformly sample an action from the conformal prediction set rather than
executing the controller's preferred action, and follow the resulting
trajectory until termination. Randomly sampled actions maintain high task success across all controller
variants, with only small differences between them. Thus, the calibrated sets
generally contain multiple effective actions rather than a single fragile
choice, providing useful alternative deliberation paths in addition to their
finite-sample statistical guarantee.

\subsubsection{Speaker-Routing Ablation}
\label{sec:speaker-routing-ablation}
 compares action-conditional routing with generic policies while holding the discussion state and controller-selected action fixed. We measure average one-step objective improvement and normalized regret relative to the counterfactual optimal speaker. Action-conditional routing achieves the largest mean objective improvement while maintaining low regret. Random, most-disagreeing, and most-spoken routing produce negative mean improvements. Round-robin obtains marginally lower normalized regret, but substantially lower improvement, partly because near-optimal speakers can have similar outcomes when the counterfactual performance range is small. These results show that no single generic notion of speaker importance is sufficient. Matching disagreement-oriented actions to dissenting agents and evidence-seeking actions to less-participating agents yields more informative messages without jointly optimizing every action--speaker pair.

\subsubsection{Action Leave-One-Out Ablation}
\label{sec:action-leave-one-out}

We remove each communication action for an entire discussion and measure the change in task accuracy. Removing \textsc{Route} causes the largest average drop, highlighting the importance of engaging underrepresented agents, while removing \textsc{SeekNewEvidence} also substantially hurts performance, especially for LLM controllers. \textsc{Challenge} has a modest effect, consistent with its targeted role in preventing premature consensus, whereas \textsc{ClarifyDisagreement} has the smallest effect and can slightly reduce accuracy, suggesting that routing and evidence seeking often resolve disagreement indirectly.

\subsubsection{Heterogeneous-Agent Ablation}
\label{sec:heterogeneous-agent-ablation}

We replace each agent with a different language model while keeping the controller fixed, introducing variation in reasoning ability, calibration, and instruction following. Results shows that Consilience remains effective across heterogeneous compositions. Strong and cross-family agent groups generally benefit most from adaptive coordination, whereas weaker compositions remain limited by their ability to extract evidence. This suggests that Consilience can exploit complementary reasoning styles and primarily relies on observable discussion dynamics rather than model-family-specific interactions.

\subsubsection{Controller-State Feature Ablation}
\label{sec:feature-ablation}
compares MLP controllers trained with different subsets of the controller-state representation. The complete state achieves the
highest task accuracy and lowest discussion objective. Information-related features (redundancy, evidence gain, and premature consensus) produce comparable accuracy, showing that they provide the most direct signals for action selection, but their higher objective indicates less efficient or controlled discussions.

Removing communication cost from the combined belief-and-information state reduces accuracy and produces the worst objective. Conversely, belief-only features yield the shortest discussions but the lowest accuracy because uncertainty and disagreement do not reveal whether new evidence has been
introduced or whether consensus is sufficiently supported. The feature groups are therefore complementary: information features guide intervention selection, belief features characterize collective reasoning, and cost discourages unnecessarily long discussions. Combining all three provides the most reliable balance between task success and discussion quality.

\section{Conclusion}

We introduced Consilience, a closed-loop framework for adaptively selecting, certifying, and routing communication in hidden-profile multi-agent reasoning. Across two benchmarks and 12 language models, Consilience consistently improves over uncontrolled discussion and often matches or exceeds full-information baselines, demonstrating that structured communication can be as important as information access itself. While our evaluation focuses on hidden-profile tasks, future work will explore broader collaborative reasoning settings.

\bibliographystyle{aaai2027}
\bibliography{aaai2027}

\appendixbanner{Appendix}{Supplementary material: a worked hidden-profile
example, the interpretation and full proof of Proposition~1, expanded
experimental results, ablations, and the complete prompt templates.}
\appendix

We first provide an example of this failure mode, and how Consilience corrects it by
surfacing complementary private evidence, as illustrated in
Figure~\ref{fig2}. Then we clarify the interpretation, assumptions, and scope of the conformal coverage guarantee in Proposition~1, and provide its complete proof. We then present additional experimental details and expanded non-conformal results on HiddenBench and the generated travel-planning benchmark, followed by analyses of task-level generalization and heterogeneous-agent groups. Next, we report ablations of the collective-state features, communication actions, speaker routing strategy, and conformal action-sampling procedure. Finally, we provide the complete prompt templates used for belief estimation, message generation, controller action selection, and message-level evaluation, together with the generation, validation, repair, and semantic-review prompts used to construct the synthetic benchmark.

\begin{figure}[!ht]
\centering
\includegraphics[width=1\textwidth]{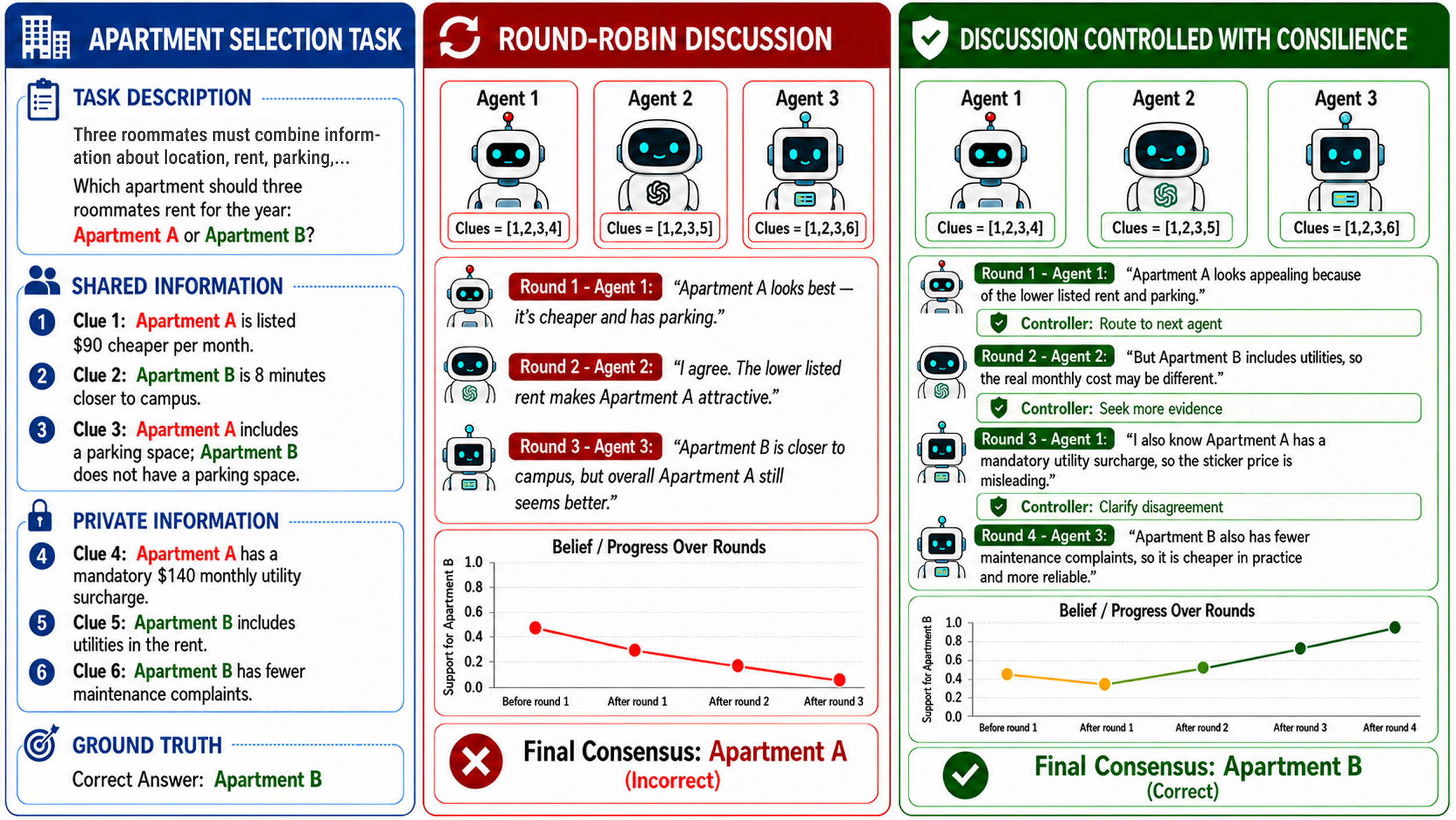} 
\caption{Three agents choose between Apartments A and B. Shared evidence initially
favors Apartment A, but complementary private clues collectively identify
Apartment B as the correct choice. Round-robin discussion repeats common
evidence further suppresses support for B (as Agent 1 has no clue about any utility surcharges for Apartment B in Round 1), and the final decision converges on A.
\textit{Consilience} instead adaptively selects interventions and speakers to surface unshared evidence, reverse premature consensus, and recover B.}
\label{fig2}
\end{figure}

\section{Interpretation, Scope, and Proof of Proposition~1}
\label{app:prop1-interpretation}

\paragraph{Coverage of the controller proposal.}
Proposition~1 provides a round-wise marginal coverage statement for the
one-step regret of the controller's proposed action. Let \(N_t\) denote
the random number of calibration trajectories that reach discussion
round \(t\), and let \(n_t\) denote its realized value. For a fixed
controller \(m\) and round \(t\), the proposition states that
\begin{equation}
\label{eq:app-proposal-coverage}
\begin{aligned}
\Pr\!\Bigl(
    &\mathcal{R}_t\!\left(
        s_t^{\mathrm{test}},
        \pi_m(s_t^{\mathrm{test}})
    \right)
    \leq \epsilon_t^m
    \\
    &\bigm|\;
    \mathcal{T}_{\mathrm{test}}
    \text{ reaches round } t,\,
    N_t=n_t
\Bigr)
\geq 1-\alpha .
\end{aligned}
\end{equation}
Equivalently, with probability at least \(1-\alpha\), the controller's
proposal belongs to the round-specific conformal action set:
\begin{equation}
\pi_m(s_t^{\mathrm{test}})
\in
\Gamma_t^m(s_t^{\mathrm{test}}).
\end{equation}

The probability in
Equation~\eqref{eq:app-proposal-coverage} is marginal over the
calibration trajectories and the new test trajectory, conditional on
the test trajectory reaching round \(t\) and on \(N_t=n_t\). It is not,
in general, a coverage guarantee conditional on one fixed realized
calibration sample.

\paragraph{Deterministic property of the executed action.}
The probabilistic coverage statement concerns the controller's raw
proposal rather than the action ultimately executed by Consilience.
The executed action is
\begin{equation}
\label{eq:accepted-action}
a_t =
\begin{cases}
\pi_m(s_t),
&
\mathcal{R}_t(s_t,\pi_m(s_t))
\leq \epsilon_t^m,
\\[3pt]
\operatorname{Select}\!\left(\Gamma_t^m(s_t)\right),
&
\mathcal{R}_t(s_t,\pi_m(s_t))
> \epsilon_t^m.
\end{cases}
\end{equation}
Provided that the conformal action set is evaluated using all
admissible counterfactual branches and that
\[
\operatorname{Select}\!\left(\Gamma_t^m(s_t)\right)
\in \Gamma_t^m(s_t),
\]
the executed action satisfies
\begin{equation}
\label{eq:executed-action-bound}
\mathcal{R}_t(s_t,a_t)
\leq \epsilon_t^m
\end{equation}
for every evaluated state \(s_t\).

Equation~\eqref{eq:executed-action-bound} is an algorithmic consequence
of the acceptance rule, rather than an additional probabilistic
conformal guarantee. If the proposal is accepted, it satisfies the
threshold by definition. If it is rejected, the fallback returns an
element of \(\Gamma_t^m(s_t)\), every member of which has regret at most
\(\epsilon_t^m\). The conformal set is nonempty because at least one
admissible action minimizes the realized counterfactual next-state loss
and therefore has zero one-step regret.

This deterministic statement assumes that the counterfactual branch
used to evaluate \(\mathcal{R}_t(s_t,a_t)\) is the branch committed to
the public trajectory. In particular, the selected speaker message and
the resulting belief updates must not be resampled after acceptance.
Otherwise, the realized next state may differ from the evaluated
counterfactual state, and
Equation~\eqref{eq:executed-action-bound} need not hold for the
resampled outcome.

\paragraph{Scope of the guarantee.}
The proposition bounds one-step regret relative to the admissible
action producing the lowest realized counterfactual next-state loss
under the surrogate discussion objective \(J\). It does not identify a
unique ground-truth communication action, guarantee that the
surrogate-optimal action improves answer correctness, establish global
policy optimality, or guarantee that the final collective answer is
correct.

The proposition also does not provide simultaneous coverage over all
rounds of a discussion trajectory. In particular, the collection of
round-wise guarantees
\[
\Pr\!\left(
    Q_t^m(s_t^{\mathrm{test}})
    \leq \epsilon_t^m
    \,\middle|\,
    \mathcal{T}_{\mathrm{test}}
    \text{ reaches round } t,\,
    N_t=n_t
\right)
\geq 1-\alpha
\]
does not imply
\[
\Pr\!\left(
    Q_t^m(s_t^{\mathrm{test}})
    \leq \epsilon_t^m
    \text{ for every reached round } t
\right)
\geq 1-\alpha.
\]
A trajectory-level guarantee would require an additional simultaneous
or sequential calibration argument, which is outside the scope of
Proposition~1.

\paragraph{Round-specific exchangeability.}
The marginal coverage statement requires exchangeability of the
calibration and test proposal scores at the round under consideration.
More precisely, conditional on \(N_t=n_t\) and on the test trajectory
reaching round \(t\), the scores
\[
Q_t^m(s_t^1),\ldots,Q_t^m(s_t^{n_t}),
Q_t^m(s_t^{\mathrm{test}})
\]
must be exchangeable. This condition holds under the proposition's
assumptions when calibration and test trajectories are generated by
the same fixed, policy-induced trajectory distribution and when the
controller and all components of the score-generating procedure are
fixed independently of the calibration sample.

Because communication control is sequential, exchangeability at a
later round does not follow automatically from exchangeability at the
initial round. In particular, if an earlier conformal replacement
changes the distribution of later test states relative to the
calibration trajectories, exchangeability of the later-round scores is
an additional assumption. The guarantee should therefore be
understood as conditional on round-specific exchangeability under the
deployed policy.

Accordingly, calibration is performed separately for every
configuration for which coverage is claimed, including the controller,
agent model, prompting configuration, admissible action space, routing
policy, termination rule, decoding procedure, and score-generating
components. Any learned component must be fixed before calibration or
trained using data independent of the calibration sample.

\paragraph{Finite-sample quantile convention.}
Let
\[
k_t =
\left\lceil
    (n_t+1)(1-\alpha)
\right\rceil.
\]
When \(k_t\leq n_t\), the threshold \(\epsilon_t^m\) is the \(k_t\)-th
smallest calibration score. When \(k_t=n_t+1\), including the case
\(n_t=0\), we use the standard convention
\[
\epsilon_t^m=+\infty.
\]
This convention yields valid but uninformative coverage because every
admissible action belongs to the resulting conformal set.

When ties occur among the pooled calibration and test scores, the
nonrandomized procedure may be conservative. The exact coverage value
\(k_t/(n_t+1)\) stated in Proposition~1 applies when \(k_t\leq n_t\)
and the pooled scores are almost surely distinct.

\paragraph{Proof of Proposition~1.}
\begin{proof}
Condition on the event
\[
\mathcal{E}_t
=
\left\{
    \mathcal{T}_{\mathrm{test}}
    \text{ reaches round }t,\,
    N_t=n_t
\right\}.
\]
For notational convenience, define
\[
Q_i = Q_t^m(s_t^i),
\qquad i=1,\ldots,n_t,
\]
and
\[
Q_{n_t+1}
=
Q_t^m(s_t^{\mathrm{test}}).
\]
By assumption,
\[
Q_1,\ldots,Q_{n_t},Q_{n_t+1}
\]
are exchangeable conditional on \(\mathcal{E}_t\).

By the definitions of \(Q_t^m\) and \(\Gamma_t^m\),
\begin{equation}
\label{eq:proposal-set-equivalence}
\pi_m(s_t^{\mathrm{test}})
\in
\Gamma_t^m(s_t^{\mathrm{test}})
\quad\Longleftrightarrow\quad
Q_{n_t+1}\leq\epsilon_t^m.
\end{equation}

First suppose that \(k_t=n_t+1\). By convention,
\(\epsilon_t^m=+\infty\), and hence
\[
\Pr\!\left(
    Q_{n_t+1}\leq\epsilon_t^m
    \,\middle|\,
    \mathcal{E}_t
\right)
=1
=
\frac{k_t}{n_t+1}.
\]

Now suppose that \(k_t\leq n_t\), so that
\(\epsilon_t^m\) is the \(k_t\)-th smallest calibration score.
Under conditional exchangeability, the standard split-conformal
order-statistic argument gives
\begin{equation}
\label{eq:rank_lower_bound}
\Pr\!\left(
    Q_{n_t+1}\leq\epsilon_t^m
    \,\middle|\,
    \mathcal{E}_t
\right)
\geq
\frac{k_t}{n_t+1}.
\end{equation}
This inequality remains valid in the presence of ties; ties may make
the nonrandomized conformal procedure conservative.

From the definition of \(k_t\),
\begin{equation}
\frac{k_t}{n_t+1}
=
\frac{
    \left\lceil
        (n_t+1)(1-\alpha)
    \right\rceil
}{
    n_t+1
}
\geq
1-\alpha.
\end{equation}
Combining this inequality with
Equation~\eqref{eq:rank_lower_bound} proves the round-wise lower
coverage bound. Equation~\eqref{eq:proposal-set-equivalence} then gives
the equivalent action-set statement.

Finally, suppose that \(k_t\leq n_t\) and that the \(n_t+1\) pooled
calibration and test scores are almost surely distinct. Conditional
exchangeability implies that the rank of the test score among the
pooled scores is uniformly distributed over
\[
\{1,\ldots,n_t+1\}.
\]
Under distinctness,
\[
Q_{n_t+1}\leq\epsilon_t^m
\]
holds if and only if the pooled rank of \(Q_{n_t+1}\) is at most
\(k_t\). Consequently,
\begin{equation}
\Pr\!\left(
    Q_{n_t+1}\leq\epsilon_t^m
    \,\middle|\,
    \mathcal{E}_t
\right)
=
\frac{k_t}{n_t+1}.
\end{equation}
Using
\[
\lceil x\rceil < x+1,
\]
we obtain
\begin{align}
\frac{k_t}{n_t+1}
&=
\frac{
    \left\lceil
        (n_t+1)(1-\alpha)
    \right\rceil
}{
    n_t+1
}
\\
&<
1-\alpha+\frac{1}{n_t+1}.
\end{align}
Together with the lower bound, this proves
\[
1-\alpha
\leq
\frac{k_t}{n_t+1}
<
1-\alpha+\frac{1}{n_t+1},
\]
completing the proof.
\end{proof}

\section{Experimental Details}
\textbf{Models.} We evaluate Consilience on 13 instruction-tuned open-weight language models spanning multiple families and scales, including Qwen3 (0.6B, 1.7B, 4B, 8B, 14B, 32B, and 235B), Qwen2.5-32B, Mistral-24B, Phi-4, Llama-3.1-8B, Llama-3.3-70B, and Gemma-3-27B. We further include GPT-4.1-mini and Claude-3-Haiku in a supplementary closed-model experiment comparing the full-information baseline against LLM-based orchestration with adaptive voting to assess whether the observed trends generalize beyond open-weight models. Unless otherwise specified, all agents within a discussion instantiate the same underlying language model. 

For LLM-based controller, the same model additionally serves as the controller responsible for selecting communication actions, whereas the learned-controller variants replace this decision process with a lightweight MLP while retaining the evaluated LLM for belief estimation, message generation, transcript analysis, and final voting.

\section{Non-Conformal Results}

\begin{table}[t]
\centering
\setlength{\tabcolsep}{3.4pt}
\renewcommand{\arraystretch}{1.08}
\caption{
Task and vote accuracy for Hiddenbench across the evaluated language models.
Each entry reports \emph{task accuracy / vote accuracy}.
Higher values are better.
}
\label{tab:hb-results}
\widetab{\begin{tabular}{lccccccc}
\toprule
\rowcolor{tblhead}
\textbf{Model} &
\textbf{Hidden Pre} &
\textbf{Hidden Post} &
\textbf{Full Info} &
\textbf{Rules} &
\textbf{MLP} &
\textbf{LLM} &
\textbf{LLM+Vote} \\
\midrule

Qwen3-0.6B
& .159/.202
& .185/.203
& .359/.364
& \textbf{.374/.391}
& .287/.295
& .313/.306
& .297/.314 \\

Qwen3-1.7B
& .062/.095
& .092/.092
& .267/.285
& .559/.563
& \textbf{.600/.599}
& .554/.561
& .564/.561 \\

Qwen3-4B
& .092/.095
& .097/.124
& .369/.406
& .826/.829
& \textbf{.831/.829}
& .826/.831
& .800/.789 \\

Qwen3-8B
& .062/.126
& .149/.175
& .631/.596
&\textbf{ .908/.913}
& .856/.871
& .877/.879
& .805/.808 \\

Qwen3-14B
& .123/.161
& .267/.266
& .774/.752
& .872/.879
& .774/.783
& .821/.841
& \textbf{.928/.928} \\

Qwen3-32B
& .108/.141
& .282/.287
& .826/.817
& .918/.900
& .841/.856
& .913/.903
& \textbf{.923/.918} \\

Qwen3-235B
& .062/.123
& .292/.306
& .862/.841
& \textbf{.928/.943}
& .856/.846
& .887/.879
& .836/.848 \\

Qwen2.5-32B
& .046/.087
& .297/.290
& .908/.881
& .897/.903
& .867/.874
& \textbf{.923/.926}
& .918/.924 \\

Llama-3.1-8B
& .077/.177
& .354/.354
& .610/.599
& \textbf{.733/.682}
& .667/.655
& .595/.572
& .600/.590 \\

Llama-3.3-70B
& .046/.123
& .436/.460
& .862/.877
& .928/.935
& .928/.926
& \textbf{.949/.953}
& .908/.914 \\

Phi-4
& .077/.150
& .195/.221
& .877/.819
& .923/.924
& .826/.831
& .903/.895
& \textbf{.944/.937} \\

Mistral-24B
& .072/.170
& .467/.485
& .908/.859
& \textbf{.923/.924}
& .841/.848
& .903/.914
& .923/.925 \\

Gemma-3-27B
& .097/.184
& .349/.352
& .846/.821
& \textbf{.903/.883}
& .872/.863
& .856/.848
& .897/.904 \\

DeepSeek-V4-Flash
& .092/.185
& .384/.395
& .907/.881
& \textbf{.953/.957}
& .784/.779
& .953/.937
& .908/.889 \\

DeepSeek-V4-Pro
& .077/.158
& .431/.474
& .954/.933
& \textbf{1.00/1.00}
& .939/.921
& .969/.968
& .969/.972 \\

\bottomrule
\end{tabular}}
\end{table}

\begin{table}[t]
\centering
\setlength{\tabcolsep}{3.4pt}
\renewcommand{\arraystretch}{1.08}
\caption{
Task and vote accuracy for the generated benchmark across the evaluated language models.
Each entry reports \emph{task accuracy / vote accuracy}. Higher values are better.}
\label{tab:nb-results}
\widetab{\begin{tabular}{lccccccc}
\toprule
\rowcolor{tblhead}
\textbf{Model} &
\textbf{Hidden Pre} &
\textbf{Hidden Post} &
\textbf{Full Info} &
\textbf{Rules} &
\textbf{MLP} &
\textbf{LLM} &
\textbf{LLM+Vote} \\
\midrule

Qwen3-0.6B
& .228/.240
& .236/.240
& .204/.230
& .228/.230
& .220/.220
& .209/.210
& \textbf{.248/.240} \\

Qwen3-1.7B
& .204/.210
& .272/.270
& .248/.240
& \textbf{.312/.310}
& .264/.260
& .252/.250
& .264/.260 \\

Qwen3-4B
& .224/.230
& .292/.290
& .292/.270
& .456/.440
& .428/.430
& \textbf{.480/.480}
& .476/.470 \\

Llama-3.1-8B
& .224/.230
& .312/.310
& .280/.270
& .509/.530
& .384/.370
& \textbf{.524/.520}
& .504/.490 \\

Qwen3-8B
& .248/.230
& .324/.300
& .288/.270
& .684/.690
& .528/.540
& .652/.650
& \textbf{.716/.710} \\

Qwen3-14B
& .256/.240
& .400/.400
& .336/.310
& .696/.690
& .544/.540
& .684/.690
& \textbf{.744/.750} \\

Mistral-24B
& .292/.280
& .576/.570
& .432/.370
& .872/.870
& .496/.480
& .624/.620
& \textbf{.832/.840} \\

Gemma-3-27B
& .244/.250
& .730/.720
& .444/.400
& .806/.800
& .549/.560
& .750/.740
& \textbf{.858/.850} \\

Qwen3-32B
& .256/.240
& .668/.660
& .376/.360
& .660/.650
& .384/.370
& .592/.580
& \textbf{.716/.710} \\

Qwen2.5-32B
& .284/.280
& .684/.670
& .496/.440
& .884/.890
& .752/.740
& .880/.880
& \textbf{.900/.900} \\

Llama-3.3-70B
& .276/.270
& .728/.720
& .520/.440
& .889/.890
& .848/.830
& \textbf{.924/.920}
& .912/.910 \\

Qwen3-235B
& .308/.270
& .756/.760
& .472/.430
& \textbf{.857/.860}
& .688/.690
& .809/.810
& .848/.840 \\

\bottomrule
\end{tabular}}
\end{table}

\paragraph{HiddenBench.}
Table~\ref{tab:hb-results} shows that access to distributed evidence
alone is insufficient without effective coordination. Averaged across the 15
models, hidden pre-discussion achieves only \(0.083\) task accuracy, while
uncontrolled round-robin discussion improves this to \(0.285\). Providing all
evidence directly increases average accuracy to \(0.731\), confirming that the
main difficulty arises from recovering and integrating information distributed
across agents. All four Consilience controllers exceed the full-information
reference on average: Rules achieves \(0.843\) task accuracy, followed by LLM
at \(0.816\), LLM+Vote at \(0.815\), and MLP at \(0.785\). The rule-based
controller therefore provides the strongest average performance, although the
best controller varies across models. In particular, the LLM-based controllers
are competitive for stronger models, and adaptive voting produces the best
results for models such as Qwen3-14B and Phi-4. Task and vote accuracy are also
closely aligned for all controllers, indicating that improvements in the final
collective answer generally correspond to broader agreement among the
individual agents rather than being produced only by plurality aggregation.

Supplementary experiments with frontier API models show the same trend:
Consilience with LLM+Vote matches or exceeds full-information reasoning,
achieving \(0.938\) versus \(0.877\) on openai's gpt models and \(0.862\) versus \(0.646\) on Claude-3-Haiku.

\paragraph{Generated benchmark.}
Table~\ref{tab:nb-results} presents a different performance
profile. Hidden pre-discussion obtains an average task accuracy of \(0.254\),
while uncontrolled discussion increases it to \(0.498\). Interestingly, hidden
post-discussion outperforms the full-information condition, which achieves only
\(0.366\). This suggests that merely placing all evidence in one context does
not guarantee that the model will identify and combine the relevant facts;
interaction can help expose, repeat, and reconcile evidence that may otherwise
be overlooked. Among the controlled methods, LLM+Vote achieves the highest
average task accuracy at \(0.668\), followed closely by Rules at \(0.654\),
LLM at \(0.615\), and MLP at \(0.507\). Relative to uncontrolled discussion,
LLM+Vote improves task accuracy by \(0.170\), whereas the MLP improves it by
only \(0.009\). Adaptive termination is therefore especially useful on this
benchmark, where continued discussion may reinforce an incorrect answer after
the relevant evidence has already been surfaced.

Performance on the generated benchmark also depends strongly on the capability
of the participating model. The smallest Qwen3 models receive limited benefit
from communication control, whereas Qwen2.5-32B and Llama-3.3-70B exceed
\(0.88\) task accuracy under the strongest controllers. The MLP controller is
less consistent under this setting, performing substantially below Rules and
the LLM-based controllers for models such as Mistral-24B and Qwen3-32B. This
pattern suggests that the fixed learned mapping from collective-state features
to actions is more sensitive to changes in benchmark structure or discussion
dynamics, while an LLM controller can adapt its action choice using the
current transcript and beliefs.

\paragraph{Comparison across benchmarks.}
Across the 12 models shared by both tables, controller performance is generally
lower on the generated benchmark. Relative to HiddenBench, average task
accuracy decreases by \(0.160\) points for Rules, \(0.261\) points for MLP, and
\(0.170\) points for LLM, but by only \(0.115\) points for LLM+Vote. Thus,
LLM+Vote exhibits the smallest cross-benchmark degradation, while the MLP
exhibits the largest. Together, these results indicate that explicit
communication control consistently improves hidden-profile reasoning, but the
most effective controller depends on the discussion distribution: simple
state-dependent rules are particularly effective on HiddenBench, whereas
transcript-aware control with adaptive voting transfers more robustly to the
generated benchmark.

\subsection{Generalization to Held-Out Tasks}
\label{sec:heldout-mlp}

The standard MLP controller is trained using exhaustive action outcomes collected from all benchmark tasks. It therefore measures how accurately a lightweight learned controller can imitate the one-step action oracle within the task distribution used to generate its supervision. However, because the same tasks contribute to both controller training and downstream evaluation, this condition should be interpreted primarily as an in-domain oracle-imitation diagnostic rather than as a strict test of task-level generalization.

We therefore construct a separate held-out evaluation. HiddenBench is partitioned into a five-task training subset and a disjoint 60-task test set. For each training task, the exhaustive action explorer follows the current discussion trajectory and, at every reached state, executes each of the four available communication actions:
\textsc{Challenge}, \textsc{ClarifyDisagreement},
\textsc{SeekNewEvidence}, and \textsc{Route}.

The resulting training examples pair the six-dimensional collective-state representation with a four-dimensional vector containing the observed reward of every communication action. The MLP is trained only on examples generated from the five training tasks. Its feature-normalization statistics, reward-normalization statistics, and network parameters are therefore estimated exclusively from the training split. After training, the controller is frozen and evaluated on the remaining 60 tasks without any additional fitting or access to their counterfactual action outcomes. At each test-time discussion state, the MLP predicts the expected immediate reward of the four actions and executes the action with the highest predicted value. Agent prompting, speaker routing, belief estimation, and final voting remain unchanged.

This evaluation measures whether the mapping learned from abstract collective-state variables to communication actions transfers across decision problems. The controller does not observe task identifiers or raw private clues directly; it acts on quantities such as group-belief entropy, inter-agent disagreement, message redundancy, communication cost, evidence gain, and premature consensus. Consequently, successful transfer would indicate that these variables capture recurring coordination conditions rather than merely memorizing the content or trajectory of individual benchmark tasks.

Across the 13 models for which held-out results are available, the held-out MLP obtains an average task accuracy of \textbf{0.73}. The corresponding all-task MLP average is \textbf{0.77}. Thus, withholding the evaluation tasks during controller training is associated with a descriptive decrease of approximately 0.04. The relatively small aggregate gap suggests that a substantial portion of the learned action-selection behavior transfers to discussion trajectories that were not observed during training.

Transfer is particularly strong for medium and large language models. 8 of the 13 evaluated models achieve at least 0.83 held-out task accuracy. Qwen3-32B, Qwen3-235B, Llama-3.3-70B, and Mistral-24B each reach 0.86 task accuracy, while Qwen3-8B, Qwen3-14B, Qwen2.5-32B, and Gemma-3-27B reach 0.83. These results indicate that the controller trained on only a small subset of tasks can still select useful communication actions when paired with agents capable of producing stable beliefs and informative responses.

The largest degradations relative to the all-task MLP occur for Qwen3-1.7B and Phi-4, whose task accuracies decrease by 0.17 and 0.16 respectively. Qwen3-0.6B also decreases by 0.09 points. This pattern suggests that held-out transfer is less reliable when the participating agents produce noisier belief distributions or when their discussion states differ substantially from those represented in the controller-training trajectories. Because the MLP observes only the collective-state summary, changes in the quality or calibration of agent beliefs can produce a state distribution that is difficult to interpret using a controller trained from a small number of tasks.

Several models show comparable or slightly higher accuracy under the held-out controller. For example, held-out task accuracy increases by 0.06 points for Qwen3-14B, 0.07 points for Llama-3.1-8B, and 0.03 points for both Qwen3-32B and Mistral-24B. These increases should not be interpreted as evidence that using less training data improves the controller. The two conditions are evaluated over different task sets, and the deliberations themselves involve stochastic language-model outputs. Small positive differences can therefore arise from variation in task difficulty, generated messages, belief estimates, or action trajectories.

Task and vote accuracy remain closely aligned under held-out evaluation. Their mean absolute difference across models is approximately 0.02 accuracy points, and for Qwen3-8B, Qwen3-14B, and Llama-3.3-70B the two values are identical. This agreement indicates that held-out performance is generally supported by the individual agents' final judgments rather than arising only from plurality aggregation or confidence-based tie-breaking. Larger discrepancies for a few models nevertheless show that communication control can affect both the correctness of the group decision and the distribution of support among agents.

Overall, the held-out results provide evidence that the controller learns reusable coordination behavior from the collective-state representation. The performance decrease relative to the in-domain MLP is modest on average, and high accuracy is retained across most medium and large models despite training on only five tasks. At the same time, the larger losses for some weaker or behaviorally distinct models show that transfer is not uniform. The learned controller generalizes most reliably when the test-time discussion dynamics resemble the state distributions encountered during training.

For a strictly controlled estimate of the generalization gap, the all-task MLP should additionally be evaluated on the same 60 held-out tasks. The current comparison uses the 65-task result for the standard MLP and the 60-task result for the held-out MLP; consequently, the reported 0.03 and 0.04 point differences combine controller-training effects with a small difference in the evaluation task set. Evaluating both checkpoints on the identical 60-task test split would isolate the effect of withholding task-level training supervision.

\section{Heterogeneous-Agent Analysis}

Table~\ref{tab:heterogeneous} evaluates whether Consilience remains effective
when the participating agents use different underlying language models.
Overall, heterogeneous groups can achieve performance comparable to strong
homogeneous groups, but their results are substantially more sensitive to
model composition, agent assignment, and controller choice. Among the
heterogeneous four-agent configurations, the rule-based controller reaches
its highest task accuracy of \(0.966\) on both Position Rot-4 and Ladder L2,
while the LLM and LLM+Voting controllers reach \(0.948\) on Ladder L4.
These results show that model heterogeneity does not inherently prevent
successful coordination when the group contains sufficiently capable and
complementary agents.

The position-rotation experiments reveal that assigning the same four models
to different agent positions can materially alter performance. For example,
rule-based accuracy varies from \(0.879\) in the documented assignment to
\(0.966\) in Position Rot-4, while LLM+Voting varies from \(0.828\) to
\(0.931\) across the rotations. Thus, performance depends not only on which
models are present, but also on which private evidence each model receives.
Using the stronger Qwen3-32B controller further improves the documented
heterogeneous group, raising MLP accuracy from \(0.776\) to \(0.931\), LLM
accuracy from \(0.845\) to \(0.897\), and LLM+Voting accuracy from \(0.879\)
to \(0.931\). This indicates that controller capability can compensate for
some of the variability introduced by heterogeneous participants.

Group composition also has a clear effect. Configurations composed primarily
of capable models, such as Family-Qwen, Spread-Low, and Ladder L4, remain
strong across several controllers. In contrast, groups containing several
small models perform substantially worse: Star and Ladder L1 obtain only
\(0.586\) task accuracy under LLM+Voting. The Weak-Link configuration shows
that introducing a single weak agent does not necessarily cause failure,
particularly for Rules, but can reduce the effectiveness of adaptive voting.
Among homogeneous groups, Qwen3-32B with LLM+Voting achieves the highest
overall task accuracy of \(0.983\), while the other strong homogeneous groups
also remain consistently competitive. Taken together, the results suggest
that Consilience generalizes to heterogeneous teams, but reliable performance
depends on both the capability distribution within the group and the
controller's ability to route communication across differently capable
agents.

\begin{table}[t]
\centering
\small
\setlength{\tabcolsep}{4pt}
\caption{
Performance of heterogeneous and homogeneous multi-agent groups on the
corresponding HiddenBench agent-count partitions. Four-agent configurations
are evaluated on the 58-task four-agent partition, while
Documented-3Agent is evaluated on the seven-task three-agent partition.
Each controller entry reports \emph{task accuracy / vote accuracy}. The orchestration controller is
Qwen3-8B unless marked with $\dagger$, where Qwen3-32B is used.
Model abbreviations are:
Q0.6, Q1.7, Q4, Q8, Q14, and Q32 for Qwen3 at the corresponding parameter
scale; Q2.5-32 for Qwen2.5-32B; P4 for Phi-4; M24 for Mistral-24B;
G27 for Gemma-3-27B; L8 for Llama-3.1-8B; and L70 for Llama-3.3-70B.
}
\label{tab:heterogeneous}
{%
\widetab{\begin{tabular}{llcccc}
\toprule
\rowcolor{tblhead}
\textbf{Configuration} &
\textbf{Agent Models} &
\textbf{Rules} &
\textbf{MLP} &
\textbf{LLM} &
\textbf{LLM+Voting} \\
\midrule

\multicolumn{6}{l}{\textit{Heterogeneous groups}} \\
\midrule

Documented-4Agent
& Q8, Q32, P4, M24
& .879/.888
& .776/.780
& .845/.866
& .879/.888 \\

Documented-3Agent
& Q8, Q32, P4
& .714/.714
& .571/.571
& .714/.667
& .714/.667 \\

Position Rot-2
& Q32, P4, M24, Q8
& .897/.897
& .897/.897
& .828/.853
& .931/.909 \\

Position Rot-3
& P4, M24, Q8, Q32
& .914/.931
& .897/.888
& .879/.875
& .897/.914 \\

Position Rot-4
& M24, Q8, Q32, P4
& .966/.961
& .879/.879
& .897/.897
& .828/.836 \\

Controller-Strong$^\dagger$
& Q8, Q32, P4, M24
& .879/.897
& .931/.892
& .897/.901
& .931/.927 \\

Family-Cross
& L8, P4, M24, G27
& .914/.897
& .776/.772
& .845/.815
& .914/.871 \\

Family-Qwen
& Q8, Q14, Q32, Q2.5-32
& .897/.892
& .862/.853
& .914/.914
& .914/.905 \\

Spread-Low
& P4, M24, G27, Q2.5-32
& .948/.940
& .793/.797
& .845/.853
& .948/.927 \\

Spread-High
& Q0.6, Q4, Q14, Q32
& .793/.746
& .845/.810
& .741/.724
& .724/.685 \\

Weak-Link
& Q0.6, Q32, P4, M24
& .897/.853
& .828/.802
& .845/.819
& .759/.759 \\

Star
& Q32, Q0.6, Q1.7, Q4
& .759/.694
& .655/.634
& .621/.621
& .586/.556 \\

Ladder L1
& Q0.6, Q1.7, Q4, Q8
& .776/.746
& .621/.616
& .724/.711
& .586/.603 \\

Ladder L2
& Q4, Q8, Q14, P4
& .966/.966
& .897/.888
& .845/.853
& .879/.871 \\

Ladder L3
& Q14, P4, M24, Q32
& .897/.918
& .828/.828
& .793/.780
& .879/.858 \\

Ladder L4
& Q32, Q2.5-32, G27, L70
& .931/.927
& .897/.892
& .948/.948
& .948/.935 \\

\midrule
\multicolumn{6}{l}{\textit{Homogeneous groups}} \\
\midrule

Qwen3-32B
& 4$\times$Q32
& .948/.927
& .862/.862
& .897/.905
& .983/.957 \\

Mistral-24B
& 4$\times$M24
& .931/.927
& .879/.879
& .914/.905
& .948/.948 \\

Phi-4
& 4$\times$P4
& .914/.918
& .879/.879
& .931/.931
& .879/.871 \\

Qwen3-8B
& 4$\times$Q8
& .948/.944
& .810/.841
& .914/.922
& .810/.819 \\

\bottomrule
\end{tabular}}%
}
\end{table}

\section{Ablations}

\paragraph{State-feature ablation.}
Table~\ref{tab:feature-group-ablation} evaluates how different subsets of
the collective-state representation affect the learned MLP controller.
Using the complete feature set achieves the best overall trade-off, with
82.22\% task success and the lowest final objective value of 0.0239.
The information-only controller matches this success rate, but incurs
slightly higher communication cost and a substantially worse final
objective, indicating that the remaining state variables improve the
quality and efficiency of the learned policy even when they do not change
the final task accuracy.

Removing information-related features reduces performance. Both the
belief-plus-information and belief-plus-cost variants achieve 77.78\%
success, while the belief-only controller performs worst at 75.56\%.
Although the belief-only representation produces the lowest communication
cost, this reduction is accompanied by a 6.66 percentage-point drop in
success relative to the full model. Among the reduced representations,
belief plus cost yields a lower final objective than belief plus
information, suggesting that explicit cost awareness helps the controller
avoid inefficient deliberation. Overall, the results show that belief
statistics alone are insufficient, while combining belief, information,
and cost-related signals produces the most balanced controller.

\begin{table}[H]
\centering
\small
\setlength{\tabcolsep}{4pt}
\caption{State-feature ablation for Qwen3-8B. Each row retrains the MLP using
only the retained feature groups. Higher success is better; lower communication
cost and final objective are better.}
\label{tab:feature-group-ablation}
\resizebox{\columnwidth}{!}{
\begin{tabular}{lrrr}
\toprule
\rowcolor{tblhead}
\textbf{Features retained} &
\textbf{Success} \(\uparrow\) &
\textbf{Communication cost }\(\downarrow\) &
\textbf{Final objective} \(\downarrow\) \\
\midrule
\textbf{All features}
    & \textbf{82.22\%} & 0.5228 & \textbf{0.0239} \\
Information only
    & \textbf{82.22\%} & 0.5254 & 0.2277 \\
Belief + information
    & 77.78\% & 0.5285 & 0.3509 \\
Belief + cost
    & 77.78\% & 0.5253 & 0.1582 \\
Belief only
    & 75.56\% & \textbf{0.4625} & 0.2212 \\
\bottomrule
\end{tabular}
}
\end{table}

\paragraph{Action ablation.}
Table~\ref{tab:action-ablation} evaluates the contribution of each
communication action by removing it while keeping the remaining controller
unchanged. Across controllers, routing and new-evidence solicitation are the
most consequential components. Removing routing produces the largest average
drop, reducing task accuracy by 4.6 percentage points, while removing
\textsc{SeekNewEvidence} decreases average accuracy by 3.8 points. These effects
are especially pronounced for the LLM+Voting controller, where removing
routing and new-evidence solicitation reduces accuracy by 13.8 and 12.3
points, respectively. This suggests that adaptive termination is particularly
dependent on directing the discussion toward underrepresented agents and
unshared evidence before voting occurs.

The effects vary across controller types. The rule-based controller depends
most strongly on routing, whereas the standard LLM controller is most affected
by removing \textsc{SeekNewEvidence}. Removing \textsc{Challenge} causes a
small but consistent degradation for Rules, MLP, and LLM, indicating that
explicitly testing the leading hypothesis provides a modest benefit.
In contrast, removing \textsc{ClarifyDisagreement} has no average effect, and
removing voting leaves the Rules and LLM+Voting results unchanged in this
ablation. Several removals improve performance, including the removal of
\textsc{SeekNewEvidence} from the MLP controller, showing that individual
actions are not uniformly useful and may interact with the controller's
selection policy. Overall, the results indicate that the full action set is
most valuable for providing complementary interventions, with routing and
new-evidence acquisition contributing the strongest aggregate gains.

\begin{table}[t]
\centering
\caption{Action ablation study on HiddenBench. Each row removes one
communication action from the full controller. Entries report task accuracy
(\%), with the change relative to the corresponding full action set shown in
parentheses. Negative values indicate performance degradation.}
\label{tab:action-ablation}

\small
\setlength{\tabcolsep}{8pt}

\begin{tabular*}{\textwidth}{@{\extracolsep{\fill}}lccccc@{}}
\toprule
\rowcolor{tblhead}
\textbf{Configuration} &
\textbf{Rules} &
\textbf{MLP} &
\textbf{LLM} &
\textbf{LLM+Vote} &
\textbf{Average} \\
\midrule

Full
& 92.3
& 87.7
& 89.2
& 83.1
& \textbf{88.1} \\

-- Challenge
& 90.8 ($-1.5$)
& 86.2 ($-1.5$)
& 87.7 ($-1.5$)
& 83.1 (0.0)
& 86.9 ($-1.2$) \\

-- Clarify
& 93.8 (+1.5)
& 87.7 (0.0)
& 89.2 (0.0)
& 81.5 ($-1.5$)
& 88.1 (0.0) \\

-- New Evidence
& 92.3 (0.0)
& 92.3 (+4.6)
& 81.5 ($-7.7$)
& 70.8 ($-12.3$)
& 84.2 ($-3.8$) \\

-- Routing
& 87.7 ($-4.6$)
& 89.2 (+1.5)
& 87.7 ($-1.5$)
& 69.2 ($-13.8$)
& \textbf{83.5 ($-4.6$)} \\

-- Voting
& 92.3 (0.0)
& NA
& NA
& 83.1 (0.0)
& 87.7 (0.0) \\

\bottomrule
\end{tabular*}
\end{table}

\paragraph{Speaker-routing ablation.}
Table~\ref{tab:speaker-routing-ablation} isolates the effect of speaker
selection by holding the controller-selected communication action fixed.
Consilience's action-conditional router achieves the lowest mean objective
change, \(\Delta J=-0.0079\), indicating that its selected speakers produce
the strongest average immediate improvement in the discussion state.
Round-robin is the closest alternative, with \(\Delta J=-0.0017\), while
random routing slightly worsens the objective on average
(\(\Delta J=0.0145\)).

The normalized-regret results are less decisive. Round-robin obtains the
lowest regret at \(0.4839\), followed closely by action-conditional routing
at \(0.4868\) and random routing at \(0.4930\). Thus, although the proposed
router performs best under the mean objective-change metric, its advantage
over simple round-robin and random selection is small, and it does not
achieve the lowest normalized regret. In contrast, selecting the
most-disagreeing or most-spoken agent performs substantially worse under
both metrics, suggesting that disagreement or participation frequency alone
is insufficient for identifying the most useful speaker. Overall, the
results support action-aware routing over these stronger heuristic
alternatives, but provide only limited evidence of a clear advantage over
round-robin routing.

\paragraph{Conformal-sampling ablation.}
Table~\ref{tab:controller-comparison} compares the four controller variants
over 400 trials with Qwen3-8B. The LLM and MLP controllers achieve the highest
mean success rate, both reaching 82.75\%, despite using different action-selection
mechanisms. Their discussion lengths are also similar, averaging 9.54 and 9.87
rounds per trial, respectively. This indicates that the lightweight learned
controller can match the success of the LLM controller, although it requires
slightly more communication.

\begin{table}[H]
\centering
\setlength{\tabcolsep}{4pt}
\caption{Counterfactual speaker-routing ablation with Qwen3-8B. The
controller-selected action is held fixed and only the speaker-selection policy
is changed. Lower objective change is better, while lower normalized
regret indicates routing decisions closer to the counterfactual optimum.}
\label{tab:speaker-routing-ablation}
\resizebox{\columnwidth}{!}{
\begin{tabular}{lrr}
\toprule
\rowcolor{tblhead}
\textbf{Routing policy }&
\textbf{Mean \(\Delta J\) \(\downarrow\)} &
\textbf{Normalized regret \(\downarrow\)} \\
\midrule
\textbf{Consilience (action-conditional)}
    & \textbf{-0.0079} & 0.4868 \\
Random
    & 0.0145 & 0.4930 \\
Round-robin
    & -0.0017 & \textbf{0.4839} \\
Most-disagreeing
    & 0.1205 & 0.5735 \\
Most-spoken
    & 0.2815 & 0.7184 \\
\bottomrule
\end{tabular}
}
\end{table}

The rule-based controller terminates substantially earlier, after only 3.24
rounds on average, but its success decreases to 64.50\%. The LLM+Voting variant
also conducts shorter discussions, averaging 5.43 rounds, yet attains the
lowest success rate of 63.00\%. These results reveal a clear trade-off between
communication length and task performance: controllers that allow longer
deliberation are substantially more successful, whereas aggressive or adaptive
termination can stop the discussion before complementary private evidence has
been sufficiently integrated. Overall, conformal sampling preserves comparable
performance for the LLM and MLP controllers, while the weaker results of the
early-terminating variants suggest that calibrated action selection alone does
not compensate for premature voting.

\begin{table}[H]
\centering
\small
\caption{Comparison of controller methods over 400 trials. Using Qwen3-8B model.}
\label{tab:controller-comparison}
\begin{tabular}{lcc}
\toprule
\rowcolor{tblhead}
\textbf{Method} & \textbf{Mean Success (\%)} & \textbf{Mean Rounds/Trial} \\
\midrule
LLM         & \textbf{82.75} & 9.54 \\
MLP         & \textbf{82.75} & 9.87 \\
Rules       & 64.50 & 3.24 \\
LLM+Voting  & 63.00 & 5.43 \\
\bottomrule
\end{tabular}
\end{table}

    \onecolumn
\section{Complete Prompt Templates}
\label{app:prompts}

This section reports the complete prompt templates used in our
experiments. Text enclosed in angle brackets denotes a runtime
placeholder populated separately for each task or deliberation state.

\begin{promptbox}{Prompt for initial belief estimation}
You are <AGENT_NAME>, one participant in a multi-agent group decision task.

Task description:
<TASK_DESCRIPTION>

Possible answers:
<POSSIBLE_ANSWERS>

Shared information visible to everyone:
- <SHARED_INFORMATION_ITEM_1>
- <SHARED_INFORMATION_ITEM_2>
- ...

Information visible only to you:
- <PRIVATE_INFORMATION_ITEM_1>
- <PRIVATE_INFORMATION_ITEM_2>
- ...

Rules:
- You may use the information visible only to you.
- Share decision-relevant evidence when it helps the group.
- Do not invent facts.
- Be concise.

Public transcript so far:
(empty transcript)

Give your current belief over the possible answers. Return JSON only:
{
  "belief": {"OPTION": 0.5},
  "best_answer": "one option exactly as written",
  "confidence": 0.0,
  "one_sentence_evidence": "short evidence summary",
  "needs_more_information": true
}

Include every possible answer exactly and make the probabilities sum to 1.
Possible answers: <POSSIBLE_ANSWERS>
\end{promptbox}

\begin{promptbox}{Prompt for speaker response generation}
You are <AGENT_NAME>, one participant in a multi-agent group decision task.

Task description:
<TASK_DESCRIPTION>

Possible answers:
<POSSIBLE_ANSWERS>

Shared information visible to everyone:
- <SHARED_INFORMATION_ITEM_1>
- <SHARED_INFORMATION_ITEM_2>
- ...

Information visible only to you:
- <PRIVATE_INFORMATION_ITEM_1>
- <PRIVATE_INFORMATION_ITEM_2>
- ...

Rules:
- You may use the information visible only to you.
- Share decision-relevant evidence when it helps the group.
- Do not invent facts.
- Be concise.

Public transcript so far:
<PUBLIC_TRANSCRIPT>

Controller instruction:
<ACTION_SPECIFIC_INSTRUCTION>

Write your next message to the group. Share relevant evidence, clearly state
if it rules out an option, communicate naturally, and do not refer to it as
private information.
\end{promptbox}

\begin{promptbox}{Controller instruction for \textsc{Challenge}}
The current leading option is <LEADING_OPTION>. State any evidence that could
support or challenge <LEADING_OPTION>. If your information rules out another
option, say so clearly.
\end{promptbox}

\begin{promptbox}{Controller instruction for \textsc{ClarifyDisagreement}}
Your belief differs from the group. Explain the evidence behind your view,
especially any fact that rules out an option.
\end{promptbox}

\begin{promptbox}{Controller instruction for \textsc{SeekNewEvidence}}
Share one new decision-relevant fact. Focus on facts that eliminate or support an option. Avoid repeating the transcript.
\end{promptbox}

\begin{promptbox}{Controller instruction for \textsc{Route}}
Contribute one concise piece of decision-relevant evidence. If possible,
explain which option it supports or rules out.
\end{promptbox}

\begin{promptbox}{Prompt for LLM controller action selection}
You are the consilience controller. Choose the next controller action for a
multi-agent decision task. You must choose exactly one existing action type; do not
choose a speaker and do not write the agent's message.

Available actions:
- challenge: ask an agent to test the current leading option and surface contrary evidence.
- clarify_disagreement: ask the agent whose belief differs most from the group to explain why.
- seek_new_evidence: ask the least-heard agent for one new decision-relevant fact.
- route: ask the least-heard agent for concise evidence that supports or rules out an option.

Task:
<TASK_DESCRIPTION>

Possible answers:
<POSSIBLE_ANSWERS>

Public transcript so far:
<PUBLIC_TRANSCRIPT>

Current state:
<COLLECTIVE_STATE_JSON>

Current round:
<CURRENT_ROUND>

Current agent beliefs:
<AGENT_BELIEFS_JSON>

Times each agent has spoken:
<SPEAKER_COUNTS_JSON>

Return JSON only:
{
  "action": "one of: challenge, clarify_disagreement, seek_new_evidence, route",
  "reason": "brief reason for choosing this action",
  "expected_effect": "brief description of what this should improve"
}
\end{promptbox}

\begin{promptbox}{Prompt for LLM controller with adaptive voting}
You are the consilience controller. Choose the next controller action for a
multi-agent decision task. You must choose exactly one existing action type; do not
choose a speaker and do not write the agent's message.

Available actions:
- challenge: ask an agent to test the current leading option and surface contrary evidence.
- clarify_disagreement: ask the agent whose belief differs most from the group to explain why.
- seek_new_evidence: ask the least-heard agent for one new decision-relevant fact.
- route: ask the least-heard agent for concise evidence that supports or rules out an option.
- vote: end the discussion now and collect every agent's final vote.

Task:
<TASK_DESCRIPTION>

Possible answers:
<POSSIBLE_ANSWERS>

Public transcript so far:
<PUBLIC_TRANSCRIPT>

Current state:
<COLLECTIVE_STATE_JSON>

Current round:
<CURRENT_ROUND>

Current agent beliefs:
<AGENT_BELIEFS_JSON>

Times each agent has spoken:
<SPEAKER_COUNTS_JSON>

Return JSON only:
{
  "action": "one of: challenge, clarify_disagreement, seek_new_evidence, route, vote",
  "reason": "brief reason for choosing this action",
  "expected_effect": "brief description of what this should improve"
}
\end{promptbox}

\begin{promptbox}{Prompt for message-level evidence and redundancy evaluation}
Analyze a message in a multi-agent decision conversation.

Options: <POSSIBLE_ANSWERS>
Transcript before this message:
<PUBLIC_TRANSCRIPT_BEFORE_MESSAGE>

New message:
<NEW_MESSAGE>

Return JSON only:
{
  "new_evidence_score": 0.0,
  "redundancy_score": 0.0,
  "supports_options": ["option names exactly"],
  "attacks_or_rules_out_options": ["option names exactly"],
  "is_clarification": false,
  "short_explanation": "one sentence"
}

Use scores between 0 and 1. New evidence adds task-relevant information; redundancy
means the message mostly repeats the earlier transcript.
\end{promptbox}

\subsection{Synthetic Benchmark Generation Prompts}
\label{app:synthetic-benchmark-prompts}

The additional travel-planning benchmark was generated using structured
prompting followed by deterministic validation, automatic repair, and
LLM-based semantic quality review. Text enclosed in angle brackets denotes
a value populated at runtime.

\begin{promptbox}{System prompt for synthetic task generation}
You design adversarial, realistic multi-agent travel-planning
benchmarks. Return exactly one valid JSON object and no markdown. Never reveal or hint at the
correct answer outside the declared correct_answer field. Use only the schema and operators
specified by the user.
\end{promptbox}

\begin{promptbox}{User prompt for synthetic task generation}
Create one HiddenBench-style group travel decision task.

Variation seed: <VARIATION_SEED>
Dataset position: <TASK_INDEX>
Avoid duplicating these accepted task signatures: <ACCEPTED_TASK_SIGNATURES>
Use exactly <NUMBER_OF_TRAVELERS> travelers.
Make option <TARGET_CORRECT_OPTION> the declared correct answer.
Use HiddenBench-style weighted preference aggregation:
- satisfied must = +2
- satisfied prefer = +1
- matched avoid = -1
- matched reject = -2
The declared answer must be the unique highest-scoring option. It may violate a strong preference
if it remains the best group compromise. Distractors do not need artificial, unique violation
patterns.

The task must:
- involve exactly 3 or 4 travelers with plausible roles and a concrete group-trip context;
- expose exactly four fully specified candidate plans (A, B, C, D) in public information;
- distribute private structured preferences across travelers using all tiers: must, prefer,
  avoid, reject;
- make the declared answer uniquely best under the validation profile above;
- require evidence from multiple travelers to identify the answer;
- remain solvable by literal comparison only: no outside travel facts, arithmetic ambiguity,
  subjective interpretation, compromise, or unstated assumptions;
- make every candidate feasible in ordinary real-world terms;
- give every traveler at least one preference that matches some but not all candidates;
- avoid unsafe, discriminatory, medically implausible, or internally contradictory content.

Tier semantics:
- must: matching is a strong positive preference worth +2.
- reject: matching is a strong negative preference worth -2.
- prefer: matching adds +1 utility.
- avoid: matching subtracts 1 utility.

Allowed operators: eq, neq, in, not_in, contains, contains_all, contains_any, excludes, lte, gte.
For contains/excludes operators, candidate attributes must be arrays. `in` means a scalar or any
array item belongs to the preference value list; `not_in` means no scalar/array item belongs to
that list. For lte/gte, attributes must be numbers. Every preference must reference an attribute
present in every option.

Every option attributes object must include at least these exact keys:
transport_mode, departure_window, duration_days, price_per_person_usd, hotel_category,
room_arrangement, step_free, daily_active_hours, max_walk_km_per_day, meal_options, activities.
You may add a small number of concrete keys such as arrival_time, layovers, or neighborhood.
All four options must use the same keys and comparable units.

Return this exact shape:
{
  "name": "short_unique_snake_case_name",
  "description": "shared scenario, dates or season, origin, destination, group relationship",
  "public_information": [
    "2-5 public facts that do not encode any preference or answer"
  ],
  "options": [
    {
      "id": "A",
      "title": "concise plan label",
      "attributes": {
        "transport_mode": "train",
        "departure_window": "morning",
        "duration_days": 4,
        "price_per_person_usd": 900,
        "hotel_category": "midrange",
        "room_arrangement": "two twin rooms",
        "step_free": true,
        "daily_active_hours": 7,
        "max_walk_km_per_day": 5,
        "meal_options": ["vegetarian", "seafood"],
        "activities": ["museum", "market"]
      }
    }
  ],
  "agents": [
    {
      "id": "traveler_1",
      "role": "realistic role in this group",
      "context": "one sentence explaining why the preferences are plausible",
      "preferences": [
        {
          "id": "t1_must_1",
          "tier": "must",
          "attribute": "step_free",
          "operator": "eq",
          "value": true,
          "statement": "A natural first-person statement with clear non-negotiable force."
        },
        {
          "id": "t1_reject_1",
          "tier": "reject",
          "attribute": "transport_mode",
          "operator": "eq",
          "value": "flight",
          "statement": "A natural first-person statement clearly refusing flights."
        }
      ]
    }
  ],
  "correct_answer": "one of A/B/C/D"
}

Use 2-4 preferences per traveler. Individual travelers may have only hard or only soft items, but
the complete task must use each of must, prefer, avoid, and reject. NEVER omit the "tier" key:
copy its lowercase value from the preference id. Before responding, evaluate every predicate
against A/B/C/D, calculate the aggregate score for all options, and verify that each traveler has
at least one predicate matching 1-3 options. It is acceptable for an individual traveler to rank
the declared answer first; at least two profiles must still contribute non-constant evidence to
the aggregate decision. The statement must faithfully express its predicate and tier. Do not add
evaluation, rationale, violation summaries, or extra top-level fields.
\end{promptbox}

\begin{promptbox}{Failure-conditioned regeneration instruction}
A previous blueprint was abandoned after failed repairs. Create a genuinely new blueprint and
avoid these failures:
<PRIOR_VALIDATION_ISSUES>
\end{promptbox}

\subsection{Deterministic Validation and Repair}
\label{app:synthetic-validation-repair}

Generated tasks are first evaluated by a deterministic validator. Under the
HiddenBench validation profile, a matched \texttt{must} contributes \(+2\),
a matched \texttt{prefer} contributes \(+1\), a matched \texttt{avoid}
contributes \(-1\), and a matched \texttt{reject} contributes \(-2\).
The declared answer must be the unique highest-scoring candidate. Instances
that fail validation are repaired using the following prompt.

\begin{promptbox}{Prompt for repairing a rejected synthetic task}
Repair the rejected HiddenBench travel blueprint below in place.

Dataset position: <TASK_INDEX>
Do not duplicate accepted signatures: <ACCEPTED_TASK_SIGNATURES>

Change the minimum option attributes or preferences needed to fix every validator issue. Preserve
the scenario, traveler roles, four option ids, and all already-valid content. Return the complete
repaired blueprint as one JSON object, not a patch.

For an "expected exactly one hard-feasible option, found []" error in strict mode, repair in this order:
1. Read the detailed violation list for the declared correct_answer.
2. Change its public attributes, or the conflicting preference predicate and matching statement,
   until it violates zero MUST/REJECT preferences.
3. Re-evaluate every hard predicate literally against all four options.
4. Ensure each other option retains a different non-empty violation set.
Do not merely repeat the rejected blueprint.

Required invariants after repair:
- exactly four comparable plans and exactly 3 or 4 travelers;
- every preference contains id, tier, attribute, operator, value, and statement;
- tier is exactly one lowercase value: must, prefer, avoid, or reject;
- each traveler has at least two preferences and at least one predicate that matches some but not
  all plans;
- the task collectively uses all four tiers;
- using must=+2, prefer=+1, avoid=-1, reject=-2, correct_answer is the unique
  highest-scoring option;
- an individual traveler may rank correct_answer first, but at least two profiles must contribute
  non-constant option scores;
- plans may violate MUST or REJECT items; do not force a fully hard-feasible plan;
- every preference statement exactly matches its structured predicate.

Validator issues:
<VALIDATOR_ISSUES_JSON>

Rejected blueprint:
<REJECTED_BLUEPRINT_JSON>
\end{promptbox}

\begin{promptbox}{System prompt for synthetic-task quality review}
You are a strict dataset reviewer. Check a travel hidden-profile
benchmark for semantic realism and consistency. Return exactly one valid JSON object with keys
accepted (boolean), issues (array of concise strings), and checks (object). Do not repair the task.
\end{promptbox}

\begin{promptbox}{User prompt for synthetic-task quality review}
Review this candidate benchmark after deterministic validation.

Reject it if any preference statement disagrees with its structured predicate; a plan attribute
is vague, incomparable, or implausible; a private fact leaks into public text; the declared answer
is hinted by names/order/wording; the scenario needs outside knowledge; constraints are contrived
instead of realistic; an entire traveler profile has no decision-relevant information; two options
are semantically indistinguishable; or the task could be solved from public information or one
profile alone. A realistic preference shared by all plans is allowed when that same traveler has
another preference that distinguishes plans.

Accept only if the exact four plans, realistic profiles, a unique weighted-utility answer under
must=+2, prefer=+1, avoid=-1, reject=-2, and multi-agent dependency are clear to a careful human.
Do not require every strong preference to be satisfied in HiddenBench mode; realistic compromise
plans are allowed.

The deterministic audit is authoritative for predicate evaluation, scoring, winner uniqueness,
and information dependency. Do NOT reject a task merely because any option fails a MUST, matches a
REJECT, or because one traveler locally ranks the final answer first. Do not reinterpret operators
or report option-level score violations already handled by the audit. Review only semantic
fidelity of statements to predicates, realism, ambiguity, leakage, and unsupported outside facts.

Deterministic audit:
<DETERMINISTIC_AUDIT_JSON>

Blueprint:
<CANDIDATE_BLUEPRINT_JSON>
\end{promptbox}

\end{document}